%% file: main_camera.tex
\documentclass{article}

    \usepackage[final]{neurips_2022}

\include{defs}

\usepackage[utf8]{inputenc} %
\usepackage[T1]{fontenc}    %
\usepackage{hyperref}       %
\usepackage{url}            %
\usepackage{booktabs}       %
\usepackage{amsfonts}       %
\usepackage{nicefrac}       %
\usepackage{microtype}      %
\usepackage{xcolor}         %
\usepackage{bbm}            %
\usepackage{wrapfig}        %
\usepackage{float}          %
\usepackage{caption}
\usepackage{subcaption}

\usepackage{amsmath}
\usepackage{amsthm}
\usepackage{amsfonts}
\usepackage{bm}
\usepackage{algorithmic}
\usepackage{algorithm}

\usepackage{amssymb}
\usepackage{listings}
\usepackage{mathtools}
\usepackage{selectp}

\usepackage[toc,page,header]{appendix}
\usepackage{minitoc}

\newcommand{\D}{\mathrm{D}}
\newcommand{\KL}{\mathrm{KL}}
\newcommand{\BE}{\mathbb{E}}
\newcommand{\mc}{\mathcal}
\newcommand{\norm}[1]{\left\lVert#1\right\rVert}

\newtheorem{proposition}{Proposition}[section]
\newtheorem{theorem}{Theorem}[section]

\newtheorem{lemma}[theorem]{Lemma}
\theoremstyle{definition}
\newtheorem{definition}[theorem]{Definition}

\newtheorem{example}{Example}

\newcommand{\para}[1]{\textbf{#1}}

\title{Offline Goal-Conditioned Reinforcement Learning \\ via $f$-Advantage Regression}

\author{%
Yecheng Jason Ma, Jason Yan, Dinesh Jayaraman, Osbert Bastani \\
University of Pennsylvania \\ 
\vspace{0.15cm}
{\tt\small \{jasonyma, jasyan, dineshj, obastani\}@seas.upenn.edu}  \\ 
\textbf{\url{https://jasonma2016.github.io/GoFAR/}}
}

\begin{document}

\maketitle

\begin{abstract}
Offline goal-conditioned reinforcement learning (GCRL) promises general-purpose skill learning in the form of reaching diverse goals from purely offline datasets. We propose $\textbf{Go}$al-conditioned $f$-$\textbf{A}$dvantage $\textbf{R}$egression (GoFAR), a novel regression-based offline GCRL algorithm derived from a state-occupancy matching perspective; the key intuition is that the goal-reaching task can be formulated as a state-occupancy matching problem between a dynamics-abiding imitator agent and an expert agent that directly teleports to the goal. In contrast to prior approaches, GoFAR does not require any hindsight relabeling and enjoys uninterleaved optimization for its value and policy networks. These distinct features confer GoFAR with much better offline performance and stability as well as statistical performance guarantee that is unattainable for prior methods. Furthermore, we demonstrate that GoFAR's training objectives can be re-purposed to learn an agent-independent goal-conditioned planner from purely offline source-domain data, which enables zero-shot transfer to new target domains. Through extensive experiments, we validate GoFAR's effectiveness in various problem settings and tasks, significantly outperforming prior state-of-art. Notably, on a real robotic dexterous manipulation task, while no other method makes meaningful progress, GoFAR acquires complex manipulation behavior that successfully accomplishes diverse goals.

\end{abstract}

\doparttoc %
\faketableofcontents %

\section{Introduction}
Goal-conditioned reinforcement learning~\cite{kaelbling1993learning, schaul2015universal, plappert2018multi} (GCRL) aims to learn a repertoire of skills in the form of reaching distinct goals. \textit{Offline} GCRL~\cite{chebotar2021actionable, yang2022rethinking} is particularly promising because it enables learning general goal-reaching policies from purely offline interaction datasets without any environment interaction~\cite{levine2020offline, lange2012batch}, which can be expensive in the real-world. As offline datasets contain diverse goals and become increasingly prevalent~\cite{dasari2019robonet, kalashnikov2021mt, chebotar2021actionable}, policies learned this way can acquire a large set of useful primitives for downstream tasks~\cite{lynch2020learning}.
A central challenge in GCRL is the sparsity of reward signal~\cite{andrychowicz2017hindsight}; without any additional knowledge about the environment, an agent at a state typically only accrues positive binary reward when the state lies within the goal neighborhood. This sparse reward problem is exacerbated in the offline setting, in which the agent cannot explore the environment to discover more informative states about desired goals. Therefore, designing an effective offline GCRL algorithm is a concrete yet challenging path towards general-purpose and scalable policy learning.

In this paper, we present a novel offline GCRL algorithm, \textbf{Go}al-conditioned $\textbf{f}$-\textbf{A}dvantage \textbf{R}egression (GoFAR), first casting GCRL as a state-occupancy matching~\cite{lee2019efficient, ma2022smodice} problem and then deriving a regression-based policy objective. In particular, GoFAR begins with the following goal-conditioned state-matching objective:
\begin{equation}
    \label{eq:state-matching-objective}
    \min_\pi \D_\KL(d^\pi(s;g) \| p(s;g)) 
\end{equation}
where $d^\pi(s;g)$ is the goal-conditioned state-occupancy distribution of policy $\pi$ and $p(s;g)$ is the distribution of states that satisfy a particular goal $g$. This objective casts goal-conditioned offline RL as an imitation learning problem: a dynamics-abiding agent imitates as well as possible an expert who can teleport to the goal in one step; see Figure~\ref{figure:state-matching-concept-figure}. Posing GCRL this way is mathematically principled, as we show that this objective is equivalent to a probabilistic interpretation of GCRL that additionally encourages maximizing state entropy. More importantly, this objectives enables us to extend a state-of-art offline imitation learning algorithm~\cite{ma2022smodice} to the goal-conditioned setting and admits elegant optimization using purely offline data by considering an $f$-divergence regularized objective and exploiting ideas from convex duality theory~\cite{rockafellar-1970a, boyd2004convex}. In particular, we obtain the dual optimal value function $V^*$ from a single unconstrained optimization problem, using which we construct the optimal advantage function (we refer to this as \textit{f-advantage}) that serves as importance weighting for a regression-based policy training objective; see Figure~\ref{figure:gofar-concept-figure} for a schematic illustration.
\begin{figure*}[t!]
\includegraphics[width=\textwidth]{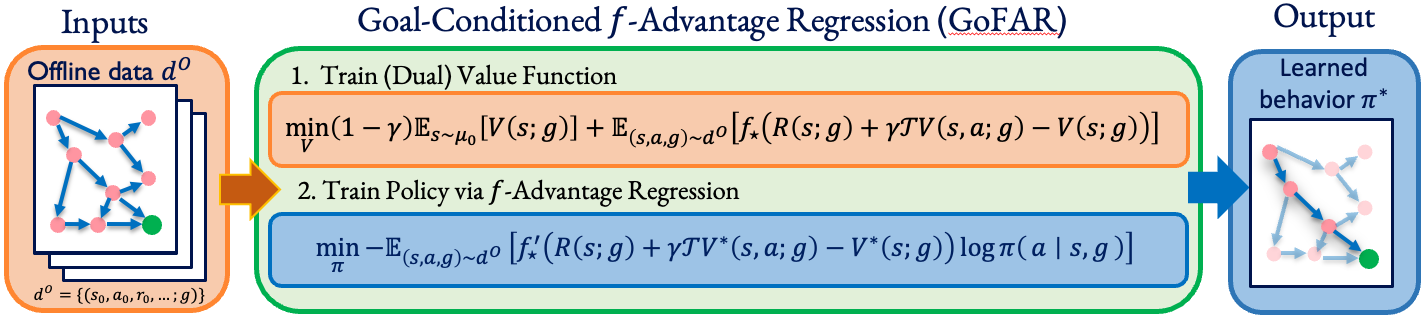}
\vspace{-0.5cm}
\caption{GoFAR schematic illustration.}
\vspace{-1cm}
\label{figure:gofar-concept-figure}
\end{figure*}

\setlength\intextsep{0pt}
\begin{wrapfigure}{r}{0.50\linewidth}
\includegraphics[width=\linewidth]{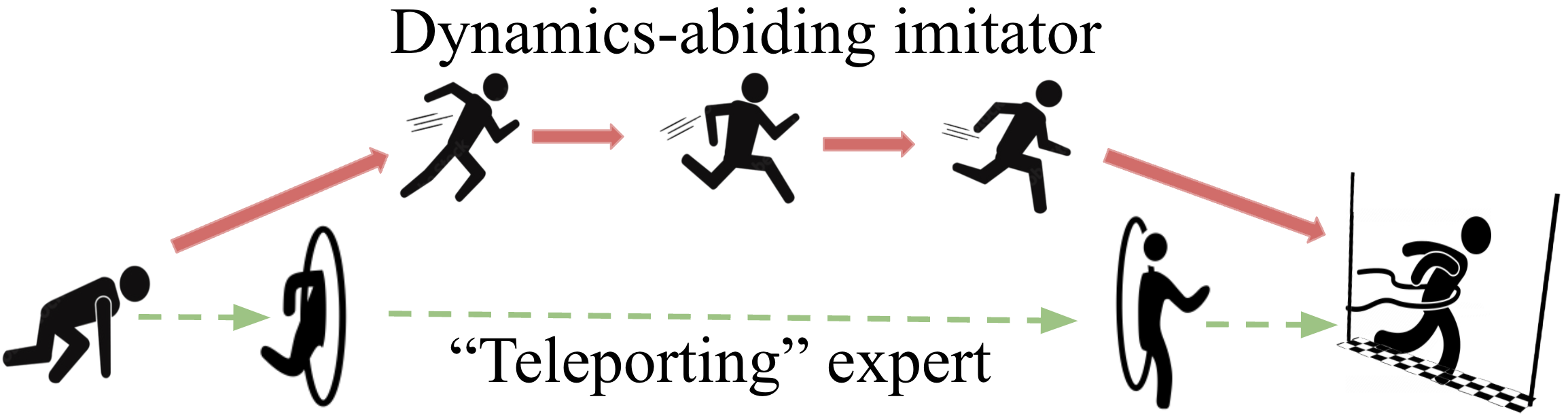}
\vspace{-0.5cm}
\caption{GCRL can be thought of imitating an expert agent that can teleport to goals.}
\label{figure:state-matching-concept-figure}
\end{wrapfigure}

There are several distinct features to our approach. First and foremost is GoFAR’s lack of goal relabeling. Hindsight goal relabeling~\cite{andrychowicz2017hindsight} (known as HER in the literature) relabels trajectory goals to be states that were actually achieved instead of the originally commanded goals. This heuristic is critical for alleviating the sparse-reward problem in prior GCRL methods, but is unnecessary for GoFAR. GoFAR updates its policy as if all data is already coming from the \textit{optimal} goal-conditioned policy and therefore does not need to perform any hindsight goal relabeling. GoFAR's \textit{relabeling-free} training is of significant practical benefits. First, it enables more stable and simpler training by avoiding sensitive hyperparameter tuning associated with HER that cannot be easily performed offline~\cite{zhang2021towards}. Second, hindsight relabeling suffers from hindsight bias in stochastic environments~\cite{schroecker2021universal}, as achieved goals may be accomplished due to noise; this is exacerbated in the offline setting due to inability to collect more data. By bypassing hindsight relabeling, GoFAR promises to be more robust and scalable.

Second, as suggested in Figure~\ref{figure:gofar-concept-figure}, GoFAR's value and policy training steps are completely \textit{uninterleaved}:  we do not need to update the policy until the value function has converged. This confers greater offline training stability~\cite{kumar2019stabilizing, ma2022smodice} compared to prior works, which mostly involve alternating updates to a critic Q-network and policy network. Furthermore, it enables an algorithmic reduction of GoFAR to weighted regression~\cite{cortes2010learning}, which allows us to obtain strong finite-sample statistical guarantee on GoFAR's performance; prior regression-based GCRL approaches~\cite{ghosh2021learning, yang2022rethinking} do not enjoy this reduction and obtain much weaker theoretical guarantees. 

Finally, we show that GoFAR can also be used to learn a goal-conditioned \textit{planner}~\cite{nasiriany2019planning, chane2021goal}. The key insight lies in observing that, under a mild assumption, GoFAR's dual value function objective does not depend on the action and can learn from state-only offline data. This enables learning a \textit{goal-centric} value function and thereby a near-optimal goal-conditioned planner that is capable of zero-shot transferring to new domains of the same task that shares the goal space. In our experiments, we illustrate that GoFAR planner can indeed plan effective subgoals in a new target domain and enable a low-level controller to reach distant goals that it is not designed for.

We extensively evaluate GoFAR on a variety of offline GCRL environments of varying task complexity and dataset composition, and show that it outperforms all baselines in all settings. Notably, GoFAR is more robust to stochastic environments than methods that depend on hindsight relabeling. We additionally demonstrate that GoFAR learns complex manipulation behavior on a real-world high-dimensional dexterous manipulation task~\cite{ahn2020robel}, for which most baselines fail to make any progress. Finally, we showcase GoFAR's planning capability in a cross-robot zero-shot transfer task. To summarize, our contributions are: 
\begin{enumerate}[topsep=0pt,itemsep=0ex,partopsep=1ex,parsep=0ex]
\item GoFAR, a novel offline GCRL algorithm derived from goal-conditioned state-matching
\item Detailed technical derivation and analysis of GoFAR's distinct features: (1) relabeling-free, (2) uninterleaved optimization, and (3) planning capability
\item Extensive experimental evaluation of GoFAR, validating its empirical gains and capabilities
\end{enumerate}

\section{Related Work} 
\para{Offline Goal-Conditioned Reinforcement Learning.} One core challenge of (offline) GCRL~\cite{kaelbling1993learning, schaul2015universal} is the sparse-reward nature of goal-reaching tasks. A popular strategy is hindsight goal relabeling~\cite{andrychowicz2017hindsight} (or HER), which relabels off-policy transitions with future goals they have achieved instead of the original commanded goals. Existing offline GCRL algorithms~\cite{chebotar2021actionable, yang2022rethinking} adapt HER-based online GCRL algorithms to the offline setting by incorporating additional components conducive to offline training. \cite{chebotar2021actionable} builds on actor-critic style GCRL algorithms~\cite{andrychowicz2017hindsight} by adding conservative Q-learning~\cite{kumar2020conservative} as well as goal-chaining, which expands the pool of candidate relabeling goals to the entire offline dataset.

 Besides actor-critic methods, goal-conditioned behavior cloning, coupled with HER, has been shown to be a simple and effective method~\cite{ghosh2021learning}. \cite{yang2022rethinking} improves upon~\cite{ghosh2021learning} by incorporating discount-factor and advantage function weighting~\cite{peng2019advantage, peters2010relative}, and shows improved performance in offline GCRL. Diverging from existing literature, GoFAR does not tackle offline GCRL by adapting an online algorithm; instead, it approaches offline GCRL from a novel perspective of state-occupancy matching and derives an elegant algorithm from the dual formulation that is naturally suited for the offline setting and carries many distinct properties that make it more stable, scalable, and versatile.

\para{State-Occupancy Matching.}
State occupancy matching objectives have been shown to be effective for learning from observations~\cite{ma2022smodice, zhu2020off}, exploration~\cite{lee2019efficient}, as well as matching a hand-designed target state distribution~\cite{ghasemipour2019divergence}; occupancy matching, in general, has been explored in the imitation learning literature~\cite{ho2016generative, torabi2018generative, ghasemipour2019divergence, ke2020imitation}. Our strategy of estimating the state-occupancy ratio is related to DICE techniques for offline policy optimization~\cite{nachum2019dualdice, nachum2019algaedice, lee2021optidice, kim2022demodice, ma2022smodice}. The closest work to ours is~\cite{ma2022smodice}, which proposed $f$-divergence state-occupancy matching for offline imitation learning. Our work shows that such a state-occupancy matching approach can be applied to GCRL. Algorithmically, GoFAR differs from~\cite{ma2022smodice} in that it does not require a separate dataset of expert demonstrations; it achieves this by modifying the discriminator, and in some settings can even be used without a discriminator, whereas a discriminator is always required by~\cite{ma2022smodice}. Furthermore, we derive several new algorithmic properties (Section~\ref{section:optimal-goal-weighting}-\ref{section:goal-conditioned-planning}) that are novel and particularly advantageous for offline GCRL.

\section{Problem Formulation}
\para{Goal-Conditioned Reinforcement Learning.} We consider an infinite-horizon Markov decision process (MDP)~\cite{puterman2014markov} $\mc{M}=(S, A,R,T, \mu_0, \gamma)$ with state space $S$, action space $A$, deterministic rewards $r(s,a)$, stochastic transitions $s' \sim T(s,a)$, initial state distribution $\mu_0(s)$, and discount factor $\gamma \in (0, 1]$. A policy $\pi:S \rightarrow \Delta(A)$ outputs a distribution over actions to use in a given state. 
In goal-conditioned RL, the MDP additionally assumes a goal space $G \coloneqq \{\phi(s) \mid s \in S\}$, where the state-to-goal mapping $\phi:S \rightarrow G$ is known. Now, the reward function $r(s;g)$\footnote{In GCRL, the reward function customarily does not depend on action.} as well as the policy $\pi(a\mid s,g)$ depend on the commanded goal $g \in G$. Given a distribution over desired goals $p(g)$, the objective of goal-conditioned RL is to find a policy $\pi$ that maximizes the discounted return:
\vspace{-0.2cm}
\begin{equation}
\label{eq:gcrl-objective}
    J(\pi) \coloneqq \mathbb{E}_{g\sim p(g), s_0\sim \mu_0, a_t \sim \pi(\cdot \mid s_t, g), s_{t+1}\sim T(\cdot \mid s_t, a_t)}\left[\sum_{t=0}^{\infty} \gamma^t r(s_t; g) \right]
\end{equation}
\vspace{-0.2cm}
The \textit{goal-conditioned} state-action occupancy distribution $d^\pi(s,a;g): S \times A \times G \rightarrow [0,1]$ of $\pi$ is
\begin{equation}
\label{eq:pi-occupancies}
\begin{split}
d^\pi(s,a;g) \coloneqq (1-\gamma) \sum_{t=0}^{\infty} \gamma^t \mathrm{Pr}(s_t=s, a_t=a \mid s_0 \sim \mu_0, a_t \sim \pi(s_t;g), s_{t+1} \sim T(s_t,a_t))
\end{split}
\end{equation}
which captures the relative frequency of state-action visitations for a policy $\pi$ conditioned on goal $g$. The state-occupancy distribution then marginalizes over actions: $d^\pi(s;g) = \sum_a d^\pi(s,a;g)$. Then, it follows that $\pi(a\mid s,g) = \frac{d^\pi(s,a;g)}{d^\pi(s;g)}$. A state-action occupancy distribution must satisfy the \textit{Bellman flow constraint} in order for it to be an occupancy distribution for some stationary policy $\pi$: 
\begin{equation}
\label{eq:bellman-flow-constraint}
\sum_a d(s,a;g) = (1-\gamma) \mu_0(s) + \gamma \sum_{\tilde{s}, \tilde{a}}T(s\mid \tilde{s}, \tilde{a}) d(\tilde{s}, \tilde{a}; g), \qquad\forall s \in S, g\in G 
\end{equation}
We write $d^\pi(s,g) = p(g)d^\pi(s;g)$ as the joint goal-state density induced by $p(g)$ and the policy $\pi$.
Finally, given $d^\pi$, we can express the objective function~\eqref{eq:gcrl-objective} as $J(\pi) = \frac{1}{1-\gamma} \mathbb{E}_{(s,g) \sim d^\pi(s,g)}[r(s;g)]$. 

\para{Offline GCRL.} In offline GCRL, the agent cannot interact with the environment $\mc{M}$; instead, it is equipped with a static dataset of logged transitions $\mc{D} \coloneqq \{\tau_i\}_{i=1}^N$, where each trajectory $\tau^{(i)} = (s_0^{(i)},a_0^{(i)}, r_0^{(i)}, s_1^{(i)},...; g^{(i)})$ with $s_0^{(i)} \sim \mu_0$ and $g^{(i)}$ is the commanded goal of the trajectory. Note that trajectories need not to be generated from a \textit{goal-directed} agent, in which case $g^{(i)}$ can be randomly drawn from $p(g)$. We denote the empirical goal-conditioned state-action occupancies of $\mc{D}^O$ as $d^O(s,a;g)$.

\section{Goal-Conditioned $f$-Advantage Regression}
In this section, we introduce Goal-conditioned $f$-Advantage Regression (GoFAR). We first derive the algorithm in full (Section~\ref{section:gofar-algorithm}), then delve deep into its several appealing properties (Section~\ref{section:optimal-goal-weighting}-\ref{section:goal-conditioned-planning}).

\subsection{Algorithm}
\label{section:gofar-algorithm}
We first show that goal-conditioned state-occupancy matching is a mathematically principled approach for solving general GCRL problems, formalizing the teleportation intuition in the introduction.
\begin{proposition}
\label{proposition:gcrl-objective}
Given any $r(s;g)$, for each $g$ in the support of $p(g)$, define $p(s;g) = \frac{e^{r(s;g)}}{Z(g)}$, where $Z(g) := \int e^{r(s;g)} ds$ is the normalizing constant. Then, the following equality holds: 
\begin{equation}
    -\D_\KL(d^\pi(s;g) \| p(s;g)) + C = (1-\gamma)J(\pi) + \mathcal{H}(d^\pi(s;g)) 
\end{equation}
where $J(\pi)$ is the GCRL objective (Eq.~\eqref{eq:gcrl-objective}) with reward $r(s;g)$ and $C := \mathbb{E}_{g \sim p(g)}[\log Z(g)]$. 
\end{proposition}
See Appendix~\ref{appendix:probabilistic-gcrl} for proof. This proposition states that, for any choice of reward $r(s;g)$, solving the GCRL problem with a maximum state-entropy regularization is equivalent to optimizing for the goal-conditioned state-occupancy matching objective with target distribution $p(s;g) := \frac{e^{r(s;g)}}{Z(g)}$.
Now, the key challenge with optimizing this objective offline is that we cannot sample from $d^\pi(s;g)$.
To address this issue, following~\cite{ma2022smodice}, we first derive an offline lower bound involving an $f$-divergence (see Appendix~\ref{appendix:additional-technical-background} for definition) regularization term, which subsequently enables solving this optimization problem via its dual using purely offline data: 
\begin{proposition}
\label{proposition:offline-upper-bound}
Assume\footnote{This assumption is only needed in our technical derivation to avoid division-by-zero issue.} for all $g$ in support of $p(g)$, $\forall s, d^O(s;g) > 0$ if $p(s;g) > 0$. Then, for any $f$-divergence that upper bounds the KL-divergence, 
\begin{align}
\label{eq:f-divergence-bound}
    -\D_\KL(d^\pi(s;g) \| p(s;g)) &\geq \mathbb{E}_{(s,g) \sim d^\pi(s,g)}\left[\log \frac{p(s;g)}{d^O(s;g)}\right] - \D_f(d^\pi(s,a;g) \| d^O(s,a;g))
\end{align}
\end{proposition}

The RHS of \eqref{eq:f-divergence-bound} can be understood as an $f$-divergence regularized GCRL objective with reward function $R(s;g) = \log \frac{p(s;g)}{d^O(s;g)}$ (we use capital $R$ to differentiate user-chosen reward $R$ from the environment reward $r$). 
Intuitively, this reward encourages visiting states that occur more often in the ``expert'' state distribution $p(s;g)$ than in the offline dataset, and the $f$-divergence regularization then ensures that the learned policy is supported by the offline dataset.
This choice of reward function can be estimated in practice by training a discriminator~\cite{goodfellow2014generative} $c: S \times G \rightarrow (0,1)$ using the offline data:
\begin{equation}
\label{eq:discriminator-training}
\min_c \BE_{g\sim p(g)} \left[ \BE_{p(s;g)}\left[\log c(s, g) \right] + \BE_{d^O(s;g)}\left[\log 1-c(s, g) \right]\right]
\end{equation}
We can in fact obtain a looser lower bound that does not require training a discriminator (see~\ref{appendix:proof} for a derivation):
\begin{equation}
        \label{eq:f-divergence-bound-binary}
   -\D_\KL(d^\pi(s;g) \| p(s;g)) \geq \mathbb{E}_{(s,g) \sim d^\pi(s,g)}\left[\log p(s;g)\right] - \D_f(d^\pi(s,a;g) \| d^O(s,a;g))
\end{equation}
Because $p(s;g) \propto e^{r(s;g)}$, we may substitute $R(s;g) := r(s;g)$ (when the offline dataset contains reward labels) for $\log p(s;g)$ and bypass having to train a discriminator for reward.

Now, for either choice of lower bound, we may pose the optimization problem with respect to valid choices of state-action occupancies directly, introducing the Bellman flow constraint~\eqref{eq:bellman-flow-constraint}:
\begin{equation}
    \label{eq:f-divergence-primal-problem} 
    \begin{split}
        \max_{d(s,a;g)\geq 0} & \quad\mathbb{E}_{(s,g) \sim d(s,g)}\left[r(s;g)\right] - \D_f(d(s,a;g) \| d^O(s,a;g)) \\ 
        (\mathrm{P})\quad \quad \mathrm{s.t.} & \quad \sum_a d(s,a;g) = (1-\gamma) \mu_0(s) + \gamma \sum_{\tilde{s}, \tilde{a}}T(s\mid \tilde{s}, \tilde{a}) d(\tilde{s}, \tilde{a}; g), \forall s \in S, g\in G 
    \end{split}
\end{equation}

This reformulation does not solve the fundamental problem that~\eqref{eq:f-divergence-primal-problem} still requires sampling from $d(s;g)$; however, it has now written the problem in a way amenable to simplification using tools from convex analysis. 
Now, we show that its dual problem can be reduced to an \textit{unconstrained} minimization problem over the dual variables which serve the role of a value function; importantly, the optimal solution to the dual problem can be used to directly retrieve the optimal primal solution:

\begin{proposition}
\label{proposition:dual}
The dual problem to~\eqref{eq:f-divergence-primal-problem} is 
\begin{equation}
\label{eq:V-problem}
     (\mathrm{D}) \quad \min_{V(s,g)\geq0}  (1-\gamma)\BE_{(s,g)\sim \mu_0,p(g)}[V(s;g)] +\BE_{(s,a,g) \sim d^O}\left[f_\star \left(R(s;g) + \gamma \mathcal{T}V(s,a;g) - V(s;g)\right)\right],
\end{equation}
where $f_\star$ denotes the convex conjugate function of $f$, $V(s;g)$ is the Lagrangian vector, and $\mathcal{T}V(s,a;g) = \mathbb{E}_{s'\sim T(\cdot \mid s,a)}[V(s';g)]$. Given the optimal $V^*$, the primal optimal $d^*$ satisfies:
\begin{equation}
\label{eq:d-star}
    d^*(s,a;g) = d^O(s,a;g) f'_\star\left(R(s;g) + \gamma \mathcal{T}V^*(s,a;g) - V^*(s;g)\right), \forall s\in S, a\in A, g\in G
\end{equation}
\end{proposition} 

A proof is given in Appendix~\ref{appendix:proof}. Crucially, as neither expectation in~\eqref{eq:V-problem} depends on samples from $d$, this objective can be estimated entirely using offline data, making it suitable for offline GCRL. For tabular MDPs, we show that for a suitable choice of $f$-divergence, the optimal $V^*$ in fact admits closed-form solutions; see Appendix~\ref{appendix:gofar-tabular-mdp} for details. In the continuous control setting, we can optimize \eqref{eq:V-problem} using stochastic gradient descent (SGD) by parameterizing $V$ using a deep neural network.

Then, once we have obtained the optimal (resp. converged) $V^*$,
we propose learning the policy via the following supervised regression update: 
\begin{equation}
    \label{eq:f-advantage-regression}
    \max_\pi \mathbb{E}_{g \sim p(g)}\mathbb{E}_{(s,a) \sim d^O(s,a;g)}\left[\left(f'_\star(R(s;g) + \gamma \mathcal{T}V^*(s,a;g)-V^*(s;g)\right)\log \pi(a\mid s, g) \right] 
\end{equation}
We see that the regression weights are the first-order derivatives of the convex conjugate of $f$ evaluated at the \textit{dual optimal advantage}, $R(s;g) + \gamma \mathcal{T}V^*(s,a;g)-V^*(s;g)$; we refer to this weighting term as \textit{$f$-advantage}. Hence, we name our overall method \textit{Goal-conditioned $f$-Advantage Regression (GoFAR)}; an abbreviated high-level pseudocode is provided in Algorithm~\ref{alg:gofar-deep-abbreviated} and a detailed version is provided in Appendix~\ref{appendix:gofar-details}. In practice, we implement GoFAR with $\chi^2$-divergence, a choice of $f$-divergence that is stable for off-policy optimization~\cite{zhu2020off, ma2022smodice, lee2021optidice}; see Appendix~\ref{appendix:gofar-details} for details.

\setlength\intextsep{0pt}
\begin{wrapfigure}{L}{0.60\linewidth}
\begin{minipage}{0.6\textwidth}
\begin{algorithm}[H]
\caption{Goal-Conditioned $f$-Advantage Regression (Abbreviated); 3-disjoint steps}\label{alg:gofar-deep-abbreviated}
\begin{algorithmic}[1]
\STATE (Optionally) Train discriminator-based reward~\eqref{eq:discriminator-training}
\STATE Train (optimal) dual value function $V^*(s;g)$~\eqref{eq:V-problem} 
\STATE Train policy $\pi$ via $f$-Advantage Regression \eqref{eq:f-advantage-regression}
\end{algorithmic}
\end{algorithm}
\vspace{0.2cm}
\end{minipage}
\end{wrapfigure}

Note that~\eqref{eq:f-advantage-regression} forgos directly minimizing~\eqref{eq:state-matching-objective} offline, which has been found to suffer from training instability~\cite{kim2022demodice}. Instead, it naturally incorporates the primal-dual optimal solutions in a regression loss. Now, we will show that this policy objective has several theoretical and practical benefits for offline GCRL that make it particularly appealing.  

\subsection{Optimal Goal-Weighting Property}
\label{section:optimal-goal-weighting}

We show that optimizing~\eqref{eq:f-advantage-regression} automatically obtains the \textit{optimal goal-weighting distribution}. That is, GoFAR trains its policy as if all the data is coming from the optimal goal-conditioned policy for~\eqref{eq:f-divergence-primal-problem}. In particular, this property is achieved without any explicit hindsight relabeling (see Appendix~\ref{appendix:additional-technical-background} for a technical definition), a mechanism that prior works heavily depend on.
To this end, we first define $p(g\mid s,a)$ to be the conditional distribution of goals in the offline dataset conditioned on state-action pair $(s,a)$. Then, according to Bayes rule, we have that
\begin{equation}
    p(g\mid s,a) = \frac{d^O(s,a;g) p(g)}{d^O(s,a)} \Rightarrow p(g)d^O(s,a;g) = p(g\mid s,a) d^O(s,a)
\end{equation}
Using this equality, we can rewrite the policy objective~\eqref{eq:f-advantage-regression} as follows:
\begin{align}
       &\min_\pi  -\mathbb{E}_{(s,a) \sim d^O(s,a)} \mathbb{E}_{g \sim p(g \mid s,a)} \left[\left(f'_\star(R(s;g) + \gamma \mathcal{T}V^*(s,a;g)-V^*(s;g)\right)\log \pi(a\mid s, g) \right] \\ 
       \label{eq:pi-star-intermediate}
       = &\min_\pi -\mathbb{E}_{(s,a) \sim d^O(s,a)} \mathbb{E}_{g \sim \tilde{p}(g \mid s,a)} \left[\log \pi(a\mid s, g) \right]
\end{align}
where 
\begin{equation}
\label{eq:optimal-goal-weighting-distribution}
\tilde{p}(g\mid s,a) \propto p(g\mid s,a) \left(f'_\star(R(s;g) + \gamma \mathcal{T}V^*(s,a;g)-V^*(s;g)\right) = p(g\mid s,a) \frac{d^*(s,a;g)}{d^O(s,a;g)}
\end{equation}
Thus, we see that GoFAR's $f$-advantage weighting scheme is equivalent to performing supervised policy regression where goals are sampled from $\tilde{p}(g\mid s,a)$. Now, combining~\eqref{eq:optimal-goal-weighting-distribution} and Bayes rule gives 
\begin{equation}
    d^O(s,a)\tilde{p}(g\mid s,a) \propto d^O(s,a) p(g\mid s,a) \frac{d^*(s,a;g)}{\frac{p(g\mid s,a) d^O(s,a)}{p(g)}} = p(g)d^*(s,a;g) 
\end{equation}
Thus, we can replace the nested expectations in~\eqref{eq:pi-star-intermediate} and obtain that GoFAR policy update amounts to supervised regression of the state-action occupancy distribution of the \textit{optimal} policy to the regularized GCRL problem~\eqref{eq:f-divergence-primal-problem}:
\begin{equation}
\label{eq:pi-gofar}
    \pi_{\mathrm{GoFAR}} = \min_\pi -\mathbb{E}_{g \sim p(g)} \mathbb{E}_{(s,a) \sim d^*(s,a;g)} \left[\log \pi(a\mid s, g) \right]
\end{equation}

This derivation makes clear GoFAR's connection with the hindsight goal relabeling mechanism~\cite{andrychowicz2017hindsight} that is ubiquitous in GCRL: GoFAR automatically performs the optimal goal-weighting policy update without any explicit goal relabeling. Furthermore, our derivation also suggests why hindsight relabeling is sub-optimal without further assumption on the reward function~\cite{eysenbach2020rewriting}: it heuristically chooses $\tilde{p}(g\mid s,a)$ to be the empirical trajectory-wise future achieved goal distribution (i.e., HER), which generally does not coincide with the goals that the optimal policy would reach; see Appendix~\ref{appendix:her-optimality} for further discussion.

\subsection{Uninterleaved Optimization and Performance Guarantee}
An additional algorithmic advantage of GoFAR is that it \textit{disentangles} the optimization of the value network and the policy network. This can be observed by noting that GoFAR's advantage term~\eqref{eq:d-star} is computed using $V^*$, the \textit{optimal} solution of the dual problem. This has the practical significance of disentangling the value-function update~\eqref{eq:V-problem} from the policy update~\eqref{eq:f-advantage-regression}, as we do not need to optimize the latter until the former has converged. This disentanglement is in sharp contrast to prior GCRL works~\cite{yang2022rethinking, chebotar2021actionable, eysenbach2020c, andrychowicz2017hindsight}, which typically involve alternating updates to the critic Q-network and the policy network, a training procedure that has found to be unstable in the offline setting~\cite{kumar2019stabilizing}. This is unavoidable for prior works because their advantage functions are estimated using the Q-estimate of the \textit{current} policy, whereas our advantage term naturally falls out from primal-dual optimality.

The uninterleaved and relabeling-free nature of GoFAR also allows us to derive strong performance guarantees.
Because $V^*$ is fixed in~\eqref{eq:f-advantage-regression}, this policy objective amounts to a \textit{weighted} supervised learning problem. Therefore, we can extend and adapt mature theoretical results for analyzing Behavior Cloning~\cite{agarwal2019reinforcement, ross2011reduction, xu2020error} as well as finite sample error guarantees for weighted regression~\cite{cortes2010learning} to obtain statistical guarantees on GoFAR's performance with respect to the optimal $\pi^*$ for~\eqref{eq:f-divergence-primal-problem}:
\begin{theorem}
\label{theorem:gofar-bound}
Assume $\sup_{s,a,g} \frac{d^*(s,a;g)}{d^O(s,a;g)} \leq M$ and $\sup \lvert r(s,g) \rvert \leq R_{\max}$. Consider a policy class $\Pi:\{S\rightarrow \Delta(A)\}$ such that $\pi^* \in \Pi$. Then, for any $\delta \in (0,1]$, with probability at least $1-\delta$, GoFAR~\eqref{eq:pi-gofar} will return a policy $\hat{\pi}$ such that:
\begin{equation}
V^* - V^{\hat{\pi}} \leq \frac{2R_{\max}M}{(1-\gamma)^2} \sqrt{\frac{\ln(|\Pi|/\delta)}{N}} 
\end{equation}
\end{theorem}
Notably, the error shrinks as the size of the \textit{offline} data $N$ increases, requiring no dependency on access to data from the ``expert'' distribution $d^*$. This provides a theoretical basis for GoFAR's empirical scalability as we are guaranteed to obtain good results when the offline data becomes more expansive. In contrast, prior regression-based GCRL methods cannot be easily reduced to a simple weighted regression with respect to the desired goal distribution $p(g)$, 
so they only obtain weaker results under stronger assumptions on the offline data (e.g, full state-space coverage) as well as the policy; see Appendix~\ref{appendix:theoretical-comparison} for discussion.

\subsection{Goal-Conditioned Planning}
\label{section:goal-conditioned-planning}
Next, we show that GoFAR can be used to learn a goal-conditioned planner that supports zero-shot transfer to other domains of the same task. The key insight is that GoFAR's value function objective \eqref{eq:V-problem} does not depend on actions assuming deterministic transitions. This is because given a transition $(s,a,s',g) \sim d^O$, $\mathcal{T}V(s,a;g) = V(s';g)$, so we can rewrite the second term in~\eqref{eq:V-problem} as $\BE_{(s,a,s',g) \sim d^O}\left[f_\star \left(R(s;g) + \gamma V(s';g) - V(s;g)\right)\right]$, which does not depend on the action.
This property enables us to learn a \textit{goal-centric} value function that is independent of the agent's state space because the agent's actions are not relevant. Specifically, we propose the following objective: 
\begin{equation}
\label{eq:V-problem-transfer}
     \min_{V(\phi(s),g)\geq0}  (1-\gamma)\BE_{(s,g)\sim \mu_0,p(g)}[V(\phi(s);g)] +\BE_{(s, s',g) \sim d^O}\left[f_\star \left(R(s;g) + \gamma V(\phi(s'),g) - V(\phi(s),g)\right)\right]
\end{equation}
We can think of this objective as learning a value function with the inductive bias that the first operation transforms the state input to the goal space via $\phi$. Since $V$ is now independent of the agent, we can use $f$-advantage regression to instead learn a \textit{goal-conditioned planner}:
\begin{equation}
\label{eq:pi-gofar-transfer}
    \max_\pi \mathbb{E}_{g \sim p(g)}\mathbb{E}_{(s,s',g) \sim d^O}\left[\left(f'_\star(R(s;g) + \gamma V^*(\phi(s');g)-V^*(\phi(s);g)\right)\log \pi(\phi(s')\mid \phi(s), g) \right] 
\end{equation}
where $\pi$ now outputs the next subgoal $\phi(s')$ conditioned on the current achieved goal $\phi(s)$ and the desired goal $g$. In our experiments, we show how this planner can achieve hierarchical control through zero-shot transfer subgoal plans to a new target domain.

\vspace{-0.3cm}
\section{Experiments}
\vspace{-0.3cm}
 We pose the following questions and provide affirmative answers in our experiments:
\begin{enumerate}[topsep=0pt,itemsep=0ex,partopsep=1ex,parsep=0ex]
\item Is GoFAR effective for offline GCRL? What components are important for performance?
\item Is GoFAR more robust to stochastic environments than hindsight relabeling methods?
\item Can GoFAR be applied to a real-robotics system?
\item Can GoFAR learn a goal-conditioned planner for zero-shot cross-embodiment transfer?
\end{enumerate}

\subsection{Offline GCRL}
\label{section:offline-gcrl} 

\begin{wrapfigure}{r}{0.50\linewidth}
\includegraphics[width=.32\linewidth]{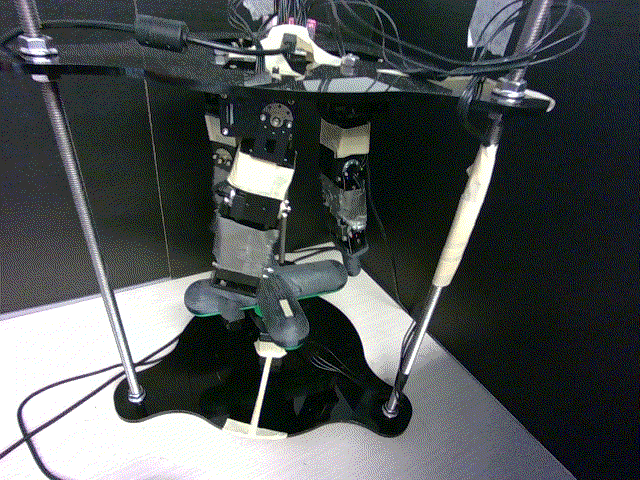} 
\includegraphics[width=.32\linewidth,trim={0 2cm 0 2.1cm},clip=True]{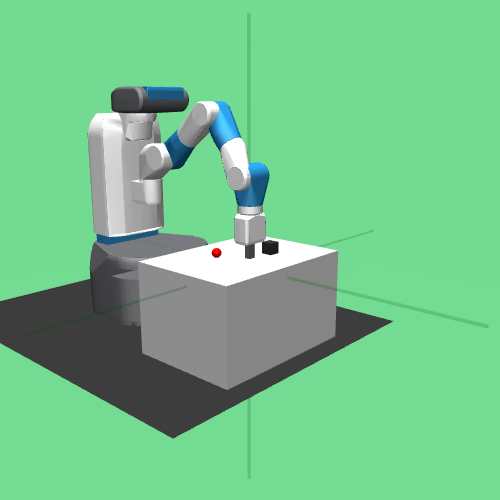} 
\includegraphics[width=.32\linewidth]{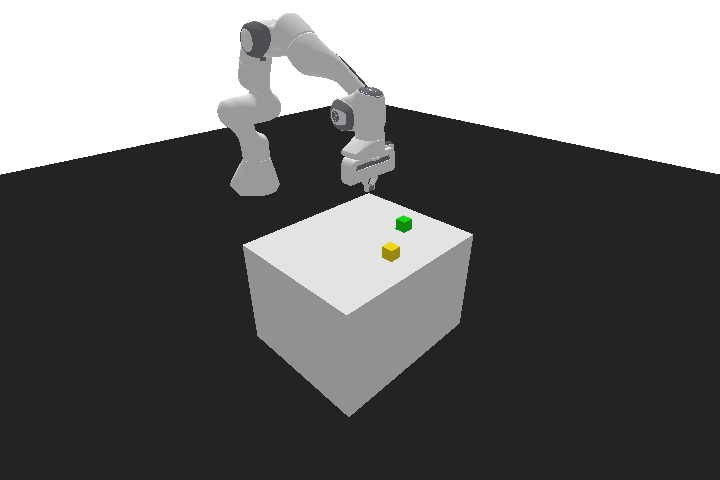}
\caption{D'Claw (left); Cross-Embodiment transfer source (middle) and target (right) domains.} 
\label{figure:fetch-panda-illustration}
\end{wrapfigure}

\para{Tasks.} In our simulation experiments, we consider six distinct environments. They include four robot manipulation environments\cite{plappert2018multi}: FetchReach, FetchPickAndPlace, FetchPush, and FetchSlide, and two dexterous manipulation environments: HandReach~\cite{plappert2018multi}, and D'ClawTurn~\cite{ahn2020robel}. All tasks use sparse reward, and their respective goal distributions are defined over valid configurations in either the robot space or object space, depending on whether the task involves object manipulation.
The offline dataset for each task is collected by either a random policy or a mixture of 90\% random policy and 10\% expert policy~\cite{ma2022smodice, kim2022demodice}, depending on whether random data provides enough coverage of the desired goal distribution. The first five tasks and their datasets are taken from~\cite{yang2022rethinking}. See Appendix~\ref{appendix:offline-gcrl} for dataset details and Appendix~\ref{appendix:task-descriptions} for detailed task descriptions and figures.

\begin{table}
\caption{Discounted Return, averaged over 10 seeds.}
\label{table:offline-gcrl-full-discounted-return}
\resizebox{\textwidth}{!}{
\begin{tabular}{l|rrr|rr}
\toprule
\multicolumn{1}{c|}{\textbf{Task}}
& \multicolumn{3}{c|}{\textbf{Supervised Learning}}
& \multicolumn{2}{c}{\textbf{Actor-Critic}} \\ 
 &  \textbf{GoFAR} (Ours) & \textbf{WGCSL} & \textbf{GCSL} & \textbf{AM} &\textbf{DDPG} \\ 
\midrule 
FetchReach  & 28.2 $\pm$ {\scriptsize0.61} & 21.9$\pm$ {\scriptsize2.13} (1.0) & 20.91 $\pm$ {\scriptsize2.78} (1.0)  & \textbf{30.1} $\pm$ {\scriptsize0.32} (0.5) & 29.8 $\pm$ {\scriptsize0.59} (0.2) \\ 
FetchPick & \textbf{19.7} $\pm$ {\scriptsize2.57} & 9.84 $\pm$ {\scriptsize2.58} (1.0) & 8.94 $\pm$ {\scriptsize3.09} (1.0) & 18.4 $\pm$ {\scriptsize3.51} (0.5) & 16.8 $\pm$ {\scriptsize3.10} (0.5) \\
FetchPush ($\star$) & \textbf{18.2} $\pm$ {\scriptsize3.00} &14.7 $\pm$ {\scriptsize2.65} (1.0) & 13.4 $\pm$ {\scriptsize3.02} (1.0)  & 14.0 $\pm$ {\scriptsize2.81} (0.5) & 12.5 $\pm$ {\scriptsize4.93} (0.5) \\
FetchSlide & 2.47 $\pm$ {\scriptsize1.44} & \textbf{2.73} $\pm$ {\scriptsize1.64} (1.0) & 1.75 $\pm$ {\scriptsize1.3}(1.0) & 1.46 $\pm$ {\scriptsize1.38} (0.5) & 1.08 $\pm$ {\scriptsize1.35} (0.5)  \\ 
\midrule 
HandReach ($\star$) & \textbf{11.5} $\pm$ {\scriptsize5.26} & 5.97 $\pm$ {\scriptsize4.81} (1.0) & 1.37 $\pm$ {\scriptsize2.21} (1.0)& 0. $\pm$ {\scriptsize0.0} (0.5) & 0.81 $\pm$ {\scriptsize1.73} (0.5)\\ 
D'ClawTurn ($\star$) & \textbf{9.34} $\pm$ {\scriptsize 3.15} & 0.0 $\pm$ {\scriptsize0.0} (1.0) & 0.0 $\pm$ {\scriptsize0.0} (1.0) & 2.82$\pm$ {\scriptsize 1.71} (1.0) & 0.0$\pm$ {\scriptsize 0.0} (0.2) \\ 
\midrule
Average Rank &  \textbf{1.5} & 3 & 4.17 & 2.83 & 4 \\ 
\bottomrule
\end{tabular}
}

\vspace{0.2cm}
\end{table}

\para{Algorithms.} 
We compare to state-of-art offline GCRL algorithms, consisting of both regression-based and actor-critic methods. The regression-based methods are: 
(1) \textbf{GCSL}~\cite{ghosh2021learning}, which incorporates hindsight relabeling in conjunction with behavior cloning to clone actions that lead to a specified goal, and (2) \textbf{WGCSL}~\cite{yang2022rethinking}, which improves upon GCSL by incorporating discount factor and advantage weighting into the supervised policy learning update. The actor-critic methods are (1) \textbf{DDPG}~\cite{andrychowicz2017hindsight}, which adapts DDPG~\cite{lillicrap2015continuous} to the goal-conditioned setting by incorporating hindsight relabeling, and (2) \textbf{ActionableModel (AM)}~\cite{chebotar2021actionable}, which incorporates conservative Q-Learning~\cite{kumar2020conservative} as well as goal-chaining on top of an actor-critic method.

We use tuned hyperparameters for each baseline on all tasks; in particular, we search the best HER ratio from $\{0.2, 0.5, 1.0\}$ for each baseline on each task separately and report the best-performing one. For GoFAR, we use identical hyperparameters as WGCSL for the shared network components and do not tune further. We train each method for 10 seeds, and each training run uses 400k minibatch updates of size 512. Complete architecture and hyperparameter table as well as additional training details are provided in Appendix~\ref{appendix:experimental-details}.

\para{Evaluations and Results.} 
We report the \textbf{discounted return} using the 
sparse binary task reward. This metric rewards algorithms that reach goals as \textit{fast} as possible and stay in the goal region thereafter. In Appendix~\ref{appendix:additional-results}, we also report the final \textbf{success rate} using the same sparse binary criterion.
Because the task reward is binary, these metrics do not take account into how \textit{precisely} a goal is being reached. Therefore, we additionally report the \textbf{final distance} to goal in Appendix~\ref{appendix:additional-results}.

The full discounted return results are shown in Table~\ref{table:offline-gcrl-full-discounted-return}; $(\star)$ indicates statistically significant improvement over the second best method under a two-sample $t$-test. The fraction inside $()$ indicates the best HER rate for each baseline. As shown, GoFAR attains the best overall performance across six tasks, and the results are statistically significant on three tasks, including the two more challenging  dexterous manipulation tasks. In Table~\ref{table:offline-gcrl-full-final-distance-appendix} in Appendix~\ref{appendix:additional-results}, we find GoFAR is also superior on the final distance metric. In other words, GoFAR reaches the goals fastest and most precisely.

\para{Ablations.} 
Recall that GoFAR achieves optimal goal-weighting, and does not require the HER heuristic that has been key to prior GCRL approaches. 
We now experimentally confirm that the baselines are not performant without HER. Conversely, we investigate whether HER can help GoFAR. Additionally, we evaluate GoFAR using the sparse binary reward, which implements the looser lower bound in~\eqref{eq:f-divergence-bound-binary}.
\setlength\intextsep{0pt}
\begin{wrapfigure}{r}{0.60\linewidth}
\centering
\includegraphics[width=\linewidth]{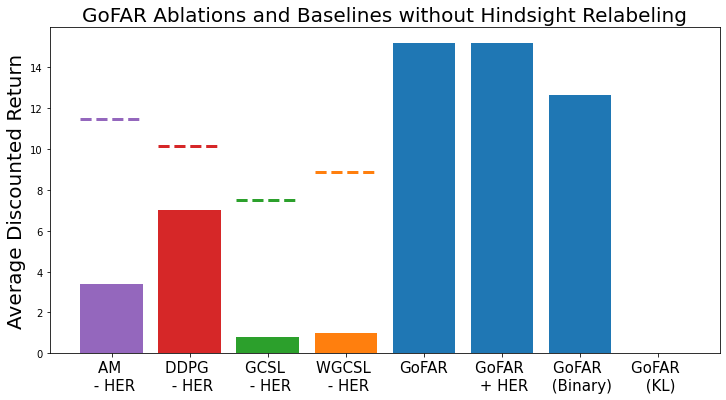}
\vspace{-0.5cm}
\caption{Offline GCRL ablation studies. While GoFAR is robust to hindsight relabeling, removing it is highly detrimental to all baselines.}
\label{figure:ablation-studies}
\end{wrapfigure}
Finally, we replace $\chi^2$-divergence with KL to understand whether the $f$-divergence lower bound in~\eqref{eq:f-divergence-bound} is necessary. The results are shown in Figure~\ref{figure:ablation-studies} (task-specific breakdown is included in Appendix~\ref{appendix:additional-results}). Compared to their original performances (colored dash lines), baselines without HER suffer severely, especially GCSL and WGCSL: both these methods in fact require $100\%$ relabeling and rely solely on HER for positive learning signals. In sharp contrast, GoFAR is also a regression-based method, yet it performs very well without HER and its performance is unaffected by adding HER. 
These findings are consistent with the theoretical results of Section~\ref{section:optimal-goal-weighting}. 
GoFAR (Binary) performs slightly worse than GoFAR, which is to be expected due to the looser bound; however, it is still better than all baselines, which use the same binary reward and are aided by HER. Again, this ablation highlights the merit of our optimization approach. Finally, GoFAR (KL) suffers from numerical instability and fails to learn, showing that our general $f$-divergence lower bound~\eqref{eq:f-divergence-bound} is not only of theoretical value but also practically significant.

\subsection{Robustness in Stochastic Offline GCRL Settings}

For this experiment, we create noisy variants of the FetchReach environment by adding varying levels of white Gaussian noise to policy action outputs. FetchReach is well-suited for this experiment because all baselines attain satisfactory performance and even outperform GoFAR in Table~\ref{table:offline-gcrl-full-discounted-return}. Furthermore, to isolate the effect of noisy actions, we use GoFAR with binary reward so that all methods now use the same reward input and perform comparatively without added noise. We consider zero-mean Gaussian noise of $\sigma$ value from $0.5, 1,$ and $1.5$. For each noise level, we re-collect offline data by executing random actions in the respective noisy environment and train on these noisy offline random data.

\begin{wrapfigure}{r}{0.50\linewidth}
\includegraphics[width=\linewidth]{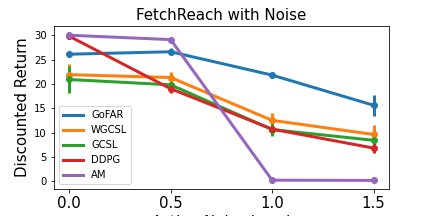}
\vspace{-0.5cm}
\caption{Stochastic environment evaluation. GoFAR is more robust to stochastic environments due to its lack of hindsight goal relabeling.}
\vspace{0.2cm}
\label{figure:stochastic-environment}
\end{wrapfigure}

The average discounted returns over different noise levels are illustrated in Figure~\ref{figure:stochastic-environment}. As shown, at noise level $0.5$, while GoFAR's performance is not effected, all baselines already exhibit degraded performance. In particular, DDPG sharply degrades, underperforming GoFAR despite better original performance. At $1.0$, the gap continues to widen, and AM notably collapses, highlighting its sensitivity to noise. This is expected as AM's relabeling mechanism also samples future goals from other trajectories and labels the selected goals with their current Q-value estimates; this procedure becomes highly noisy when the dataset already exhibits high stochasticity, thus contributing to the instability observed in this experiment. It is only at noise level $1.5$ that GoFAR's performance degradation becomes comparable to those of other methods. Thus, GoFAR's relabeling-free optimization indeed confers greater robustness to environmental stochasticity; in contrast, all baselines suffer from ``false positive'' relabeled goals due to such disturbances.

\subsection{Real-World Robotic Dexterous Manipulation}

Now, we evaluate on a real D'Claw~\cite{ahn2020robel} dexterous manipulation robot (Figure~\ref{figure:fetch-panda-illustration} (left)). The task is to have the tri-finger robot rotate the valve to a specified goal angle from its initial angle; see Appendix~\ref{appendix:task-descriptions} for detail. The offline dataset contains 400K transitions, collected by executing purely random actions in the environment, which provides sufficient coverage of the goal space.
Compared to the simulated D'Claw, an additional challenge is the inherent stochasticity on a real robot system due to imperfect actuation and hardware wear and tear. Given our conclusion in the stochastic environment experiment above, we should expect GoFAR to perform better. 

\setlength\intextsep{0pt}
\begin{wrapfigure}{r}{0.50\linewidth}
\includegraphics[width=\linewidth]{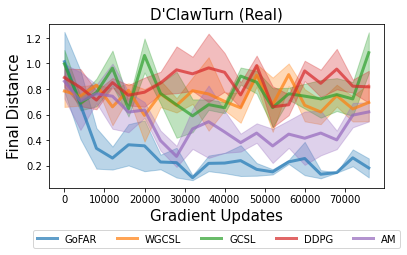}
\vspace{-0.5cm}
\caption{Real-robot results.}
\label{figure:real-robot-results}
\end{wrapfigure}

The training curve of final distances averaged over 3 seeds is shown in Figure~\ref{figure:real-robot-results}. We see that GoFAR exhibits smooth progress over training and significantly outperforms all baselines. At the end of training, GoFAR is able to rotate the valve on average within \textbf{10} degrees of the specified goal angle. In contrast, most baselines do not make any progress on this difficult task and learns policies that do not manage to turn the valve at all. AM is the only baseline that learns a non-trivial policy; however, its training fluctuates significantly, and the learned policy often executes extreme actions during evaluation that fail to reach the desired goal. We provide qualitative analysis in Appendix~\ref{appendix:additional-results-dclaw}, and videos are included in the supplement.

\subsection{Zero-Shot Transfer Across Robots}

\begin{wraptable}{r}{0.5\textwidth}
\caption{Zero-shot transfer results.}
\vspace{-0.2cm}
\label{table:zero-shot-transfer}
\resizebox{0.5\textwidth}{!}{
\begin{tabular}{l|r}
\toprule
\textbf{Method} & Success Rate\\ 
\midrule 
GoFAR Planner (Oracle)& 84\% \\ 
\midrule 
GoFAR Hierarchical Controller & 37\% \\ 
Low-Level Controller & 19\% \\ 
\bottomrule
\end{tabular}}
\end{wraptable}

For this experiment, we consider zero-shot transfer from the FetchPush task (source domain) to an equivalent planar object pushing task using the 7-DOF Franka Panda robot~\cite{gallouedec2021multi}; the environments are illustrated in Figure~\ref{figure:fetch-panda-illustration}. In the target domain, we assume a pre-programmed or learned low-level goal-reaching controller $\pi_{\mathrm{low}}(a \mid s,g)$ that can push objects over short distances but fails for longer tasks. 
For these distant goals, we train the goal-conditioned GoFAR value function $V^*(\phi(s);g)$ and planner $\pi(\phi(s')\mid \phi(s), g)$ according to~\eqref{eq:V-problem-transfer} and~\eqref{eq:pi-gofar-transfer}, to set goals for $\pi_{\mathrm{low}}$. We use the same offline FetchPush data as in Section~\ref{section:offline-gcrl}. 
We compare two approaches: (1) naively using the \textbf{Low-Level Controller}, and (2) using the GoFAR planner to guide the low-level controller (\textbf{GoFAR Hierarchical Controller}). See Appendix~\ref{appendix:zero-shot-transfer} for additional details.  We report the success rate over 100 random distant test goals in Table~\ref{table:zero-shot-transfer}. We see that naively commanding the low-level controller with these distant goals indeed results in very low success rate. When the controller is augmented with GoFAR planner, the success rate nearly doubles. Note that our planner itself is much superior: if all subgoals are perfectly reached, \textbf{GoFAR Planner (Oracle)} achieves 84\% success.  %
See Appendix~\ref{appendix:zero-shot-transfer} for additional experimental details. 
These experiments validate GoFAR's ability to zero-shot transferring subgoal plans to enhance the capabability of low-level controllers. Note that we have no access to any data from the target domain, and this planning capability naturally emerges from our training objectives and does not require any change in the algorithm. 
We show qualitative results in Appendix~\ref{appendix:additional-results-zero-shot} and videos in the supplement.

\section{Conclusion}
We have presented GoFAR, a novel regression-based offline GCRL algorithm derived from a state-occupancy matching perspective. GoFAR is relabeling-free and enjoys uninterleaved training, making it both theoretically and practically advantageous compared to prior state-of-art. Furthermore, GoFAR supports training a goal-conditioned planner with promising zero-shot transfer capability. Through extensive experiments, we have validated the practical utility of GoFAR in various challenging problem settings, showing significant improvement over prior state-of-art. We believe that GoFAR's strong theoretical foundation and empirical performance make it an ideal candidate for scalable skill learning in the real world. 

\section*{Acknowlegement}
We thank Edward Hu and Oleh Rybkin for their feedback on an earlier draft of the paper, Kun Huang for assistance on the real-robot experiment setup, and Selina Li for assistance on graphic design. This work is funded in part by an Amazon Research Award,
gift funding from NEC Laboratories America, NSF Award
CCF-1910769, NSF Award CCF-1917852 and ARO Award
W911NF-20-1-0080. The U.S. Government is authorized to
reproduce and distribute reprints for Government purposes
notwithstanding any copyright notation herein.

\newpage 
\bibliographystyle{plain}
\bibliography{references}

\newpage 
\input{neurips_checklist}

\newpage 
\appendix
\addcontentsline{toc}{section}{Appendix} %
\part{Appendix} %
These appendices contain material supplementary to submission 7863: "Offline Goal-Conditioned Reinforcement Learning via $f$-Advantage Regression"
\parttoc %

\section{Additional Technical Background}
\label{appendix:additional-technical-background}

We provide some technical definitions that are needed in our proofs in Appendix~\ref{appendix:proof} and discussions of hindsight relabeling in Section 4.1 of the main text as well as Appendix~\ref{appendix:her-optimality}.

\subsection{$f$-Divergence and Fenchel Duality}
These definitions are adapted from~\cite{ma2022smodice, nachum2020reinforcement}.

\begin{definition}[$f$-divergence]
\rm
For any continuous, convex function $f$ and two probability distributions $p,q \in \Delta(\mc{X})$ over a domain $\mc{X}$, the $f$-divergence of $p$ computed at $q$ is defined as 
\begin{equation}
\label{eq:f-divergence}
\D_f(p \| q) = \BE_{x \sim  q}\left[f\left(\frac{p(x)}{q(x)}\right)\right]
\end{equation}
\end{definition}

Some common choices of $f$-divergence includes the KL-divergence and the $\chi^2$-divergence, which corresponds to choosing $f(x) = x\log x$ and $f(x)=\frac{1}{2}(x-1)^2$, respectively. 

\begin{definition}[Fenchel conjugate]
\rm
Given a vector space $\mc{X}$ with inner-product $\langle \cdot, \cdot \rangle$, the \textit{Fenchel conjugate} $f_\star: \mc{X}_\star \rightarrow \mathbb{R}$ of a convex and differentiable function $f: \Omega \rightarrow \mathbb{R}$ is
\begin{equation}
\label{eq:fenchel-conjugate}
f_\star(y) \coloneqq \max_{x\in \mc{X}} \langle x, y \rangle - f(x)
\end{equation}
and any maximizer $x^*$ of $f_\star(y)$ satisfies $x^* = f’_\star(y)$.

For an $f$-divergence, under mild realizability assumptions~\cite{dai2016learning} on $f$, the Fenchel conjugate of $D_f(p\|q)$ at $y: \mc{X} \rightarrow \mathbb{R}$ is 
\begin{align}
\label{eq:fenchel-conjugate-f-divergence}
\D_{\star, f}(y) &= \max_{p \in \Delta(\mc{X})} \BE_{x \sim p}[y(x)] - \D_f(p\|q) \\ 
&=\BE_{x \sim q}[f_\star(y(x))]
\end{align}
and any maximizer $p^*$ of $\D_{\star, f}(y)$ satisfies $p^*(x) = q(x)f'_\star(y(x))$. These optimality conditions can be seen as extensions of the KKT-condition.
\end{definition} 

\subsection{Hindsight Goal-Relabeling}
We provide a mathematical formalism of hindsight goal relabeling~\cite{andrychowicz2017hindsight}.
\begin{definition}
Given a state $s_t$ from a trajectory $\tau = \{s_0,a_0,r_0,...,s_T;g\}$, hindsight goal-relabeling is the goal-relabeling distribution
\begin{equation}
    \label{eq:her-definition}
    p_{\mathrm{HER}}(g \mid s_t, a_t, \tau) = q[\phi(s_t), ..., \phi(s_T)]
\end{equation}
where $q$ is some categorical distribution taking values in $\{\phi(s_t), ..., \phi(s_T)\}$.
\end{definition}
That is, the relabeled goal is selected from some distribution goals that are reached in the future in the same trajectory. The most canonical choice of $q$, known as \textit{hindsight experience replay} (HER), selects $q$ to be the uniform distribution. Once a goal $\tilde{g}$ is chosen, the reward label is also re-computed using the reward function assumed by the algorithm: $r_t \coloneqq r(s_t, \tilde{g})$.

\section{Proofs}
\label{appendix:proof}
In this section, we restate propositions and theorems in the paper and present their proofs. 

\subsection{Proof of Proposition~\ref{proposition:offline-upper-bound}}
\begin{proposition}
Assume for all $g$ in support of $p(g)$, $\forall s, d^O(s;g) > 0$ if $p(s;g) > 0$. Then, for any $f$-divergence that upper bounds the KL-divergence, 
\begin{align}
    \label{eq:f-divergence-bound-appendix}
    -\D_\KL(d^\pi(s;g) \| p(s;g)) &\geq \mathbb{E}_{(s,g) \sim d^\pi(s,g)}\left[\log \frac{p(s;g)}{d^O(s;g)}\right] - \D_f(d^\pi(s,a;g) \| d^O(s,a;g)) \\ 
    \label{eq:f-divergence-bound-binary-appendix}
    &\geq \mathbb{E}_{(s,g) \sim d^\pi(s,g)}\left[\log p(s;g)\right] - \D_f(d^\pi(s,a;g) \| d^O(s,a;g))
\end{align}
\end{proposition}
\begin{proof}
We first present and prove some technical lemma needed to prove this result. The following lemmas and proofs are adapted from~\cite{ma2022smodice}; in particular, we extend these known results to the goal-conditioned setting.
\begin{lemma}
\label{lemma:state-matching-upper-bound}
For any pair of valid occupancy distributions $d_1$ and $d_2$,
we have
\begin{equation}
\D_\KL(d_1(s;g)\|d_2(s;g) \leq \D_\KL(d_1(s,a;g)\|d_2(s,a;g))
\end{equation}
\end{lemma}
\begin{proof}
This lemma hinges on proving the following lemma first.
\begin{lemma}
\begin{equation}
\D_\KL \left(d_1(s,a,s';g) \| d_2(s,a,s';g)\right) = \D_\KL\left(d_1(s,a;g) \| d_2(s,a;g) \right)
\end{equation}
\end{lemma}
\begin{proof}
\begin{align*}
    &\D_\KL \left(d_1(s,a,s';g) \| d_2(s,a,s';g)\right)\\ 
    =& \int_{S \times A \times S \times G} p(g) d_1(s,a,s';g) \log \frac{d_1(s,a;g) \cdot T(s' \mid s,a)}{d_2(s,a;g) \cdot T(s'\mid s,a)} ds' dadsdg \\ 
    =& \int_{S \times A \times S \times G} p(g) d_1(s,a,s';g) \log \frac{d_1(s,a;g)}{d_2(s,a;g)} ds' dads dg \\ 
    =& \int_{S \times A \times G} p(g) d_1(s,a;g)\log \frac{d_1(s,a;g)}{d_2(s,a;g)} dads dg\\ 
    =&\D_\KL\left(d_1(s,a;g) \| d_2(s,a;g) \right)
\end{align*}
\end{proof}
Using this result, we can prove Lemma \ref{lemma:state-matching-upper-bound}:
\begin{align*}
    &\D_\KL\left(d_1(s,a;g) \| d_2(s,a;g) \right) \\
    =& \D_\KL \left(d_1(s,a,s';g) \| d_2(s,a,s';g)\right) \\
    =& \int_{S \times A \times S \times G} p(g) d_1(s,a,s';g) \log \frac{d_1(s,a;g) \cdot T(s' \mid s,a)}{d_2(s,a;g) \cdot T(s'\mid s,a)} ds' dads dg\\ 
    =& \int_{S \times A \times S \times G} p(g) d_1(s;g) \pi_1(a\mid s,g) T(s'\mid s,a) \log \frac{d_1(s,a;g) \cdot T(s' \mid s,a)}{d_2(s,a;g) \cdot T(s'\mid s,a)} ds' dads dg\\ 
    =& \int p(g)d_1(s;g) \pi_1(a\mid s,g) T(s'\mid s,a) \log \frac{d_1(s;g)}{d_2(s;g)} ds'dadsdg\\ 
    +& \int p(g) d_1(s;g) \pi_1(a\mid s,g) T(s'\mid s,a) \log \frac{\pi_1(a\mid s,g) T(s' \mid s,a)}{\pi_2(a\mid s,g) T(s'\mid s,a)} ds'dadsdg \\
    =& \int p(g)d_1(s;g) \log \frac{d_1(s;g)}{d_2(s;g)} dsdg + \int p(g)d_1(s;g) \pi_1(a\mid s,g) \log \frac{\pi_1(a\mid s,g)}{\pi_2(a\mid s,g)} da ds dg \\
    =& \D_\KL\left(d_1(s;g) \| d_2(s;g)\right) + \D_\KL\left(\pi_1(a \mid s,g) \| \pi_2(a \mid s,g)\right) \\
    \geq& \D_\KL\left(d_1(s;g) \| d_2(s;g) \right) 
\end{align*}
\end{proof} 

Now given these technical lemmas, we have
\begin{align*}
    &D_\KL\left(d^\pi(s;g) \| p(s;g) \right) \\
    =& \int p(g) d^\pi(s;g) \log \frac{d^\pi(s;g)}{p(s;g)}\cdot \frac{d^O(s;g)}{d^O(s;g)} dsdg, \quad \text{we assume that $d^O(s;g) > 0$ whenever $p(s;g) > 0$.} \\ 
    =& \int p(g) d^\pi(s;g) \log \frac{d^O(s;g)}{p(s;g)} dsdg + \int p(g) d^\pi(s;g) \log \frac{d^\pi(s;g)}{d^O(s;g)} dsdg \\
    \leq& \BE_{(s,g) \sim d^\pi(s,g)} \left[\log \frac{d^O(s;g)}{p(s;g)} \right] + \D_\KL\left(d^\pi(s,a;g) \| d^O(s,a;g) \right)
\end{align*}
where the last step follows from Lemma \ref{lemma:state-matching-upper-bound}. Then, for any $\D_f \geq \D_\KL$, we have that
\begin{equation}
-D_\KL\left(d^\pi(s;g) \| p(s;g) \right)  \geq \BE_{(s,g) \sim d^\pi(s,g)} \left[\log \frac{p(s;g)}{d^O(s;g)} \right] - \D_f\left(d^\pi(s,a;g) \| d^O(s,a;g) \right)
\end{equation}
Then, since $\BE_{(s,g) \sim d^\pi(s,g)} \left[\log \frac{1}{d^O(s;g)} \right] \geq 0$, we also obtain the following looser bound:
\begin{equation}
-D_\KL\left(d^\pi(s;g) \| p(s;g) \right)  \geq \BE_{(s,g) \sim d^\pi(s,g)} \left[\log p(s;g) \right] - \D_f\left(d^\pi(s,a;g) \| d^O(s,a;g) \right)
\end{equation}
\end{proof}

\subsection{Proof of Proposition~\ref{proposition:dual}}
\begin{proposition}
The dual problem to~\eqref{eq:f-divergence-primal-problem} is 
\begin{equation}
\label{eq:V-problem-appendix}
     (\mathrm{D}) \quad \min_{V(s,g)\geq0}  (1-\gamma)\BE_{(s,g)\sim \mu_0,p(g)}[V(s;g)] +\BE_{(s,a,g) \sim d^O}\left[f_\star \left(r(s,g) + \gamma \mathcal{T}V(s,a;g) - V(s;g)\right)\right],
\end{equation}
where $f_\star$ denotes the convex conjugate function of $f$, $V(s;g)$ is the Lagrangian vector, and $\mathcal{T}V(s,a;g) = \mathbb{E}_{s'\sim T(\cdot \mid s,a)}[V(s';g)]$. Given the optimal $V^*$, the primal optimal $d^*$ satisfies:
\begin{equation}
\label{eq:d-star-appendix}
    d^*(s,a;g) = d^O(s,a;g) f'_\star\left(r(s,g) + \gamma \mathcal{T}V^*(s,a,g) - V^*(s,g)\right), \forall s\in S, a\in A, g\in G
\end{equation}
\end{proposition} 
\begin{proof}
We begin by writing the Lagrangian dual of the primal problem: 
\begin{equation}
\label{eq:dual-lagrangian}
\begin{split}
&\min_{V(s;g)\geq 0} \max_{d(s,a;g)\geq 0} \BE_{(s,g) \sim d(s,g)}\left[\log\left(r(s;g)\right)\right] - \D_f(d(s,a;g) \| d^O(s,a;g))\\
&+ \sum_{s,g} p(g)V(s;g)\left((1-\gamma)\mu_0(s) + \gamma \sum_{\tilde{s}, \tilde{a}}T(s\mid \tilde{s}, \tilde{a}) d(\tilde{s}, \tilde{a};g) - \sum_a d(s,a;g) \right)
\end{split}
\end{equation}
where $p(g)V(s;g)$ is the Lagrangian vector. 
Then, we note that 
\begin{equation}
    \sum_{s,g} V(s;g)\sum_{\tilde{s}, \tilde{a}}T(s\mid \tilde{s}, \tilde{a}) d(\tilde{s}, \tilde{a};g) = \sum_{\tilde{s},\tilde{a},g} d(\tilde{s},\tilde{a};g) \sum_{s}T(s \mid \tilde{s}, \tilde{a})V(s;g) = \sum_{s,a,g}d(s,a;g) \mathcal{T}V(s,a;g) 
\end{equation}
Using this, we can rewrite~\eqref{eq:dual-lagrangian} as 
\begin{equation}
\begin{split}
\min_{V(s;g)\geq0} \max_{d(s,a;g)\geq0 } (1-\gamma)\BE_{(s,g)\sim (\mu_0,p(g))}[V(s;g)] &+ \BE_{(s,a,g) \sim d}\left[\left(r(s;g) + \gamma \mathcal{T}V(s,a;g) - V(s;g)\right)\right]\\ 
&- \D_f(d(s,a;g) \| d^O(s,a;g))
\end{split}
\end{equation}
And finally,
\begin{equation}
    \label{eq:f-divergence-dual-problem}
\begin{split}
    \min_{V(s,g)\geq0}  (1-\gamma)\BE_{(s,g)\sim (\mu_0,p(g))}[V(s;g)] &+ \max_{d(s,a;g)\geq0 } \BE_{(s,a,g) \sim d}\left[\left(r(s,g) + \gamma \mathcal{T}V(s,a;g) - V(s;g)\right)\right] \\ 
    &- \D_f(d(s,a;g) \| d^O(s,a;g))
\end{split}
\end{equation}
Now, we make the key observation that the inner maximization problem in \eqref{eq:f-divergence-dual-problem} is in fact the Fenchel conjugate of $\D_f(d(s,a,g) \| d^O(s,a,g))$ at $r(s,g) + \gamma \mathcal{T}V(s,a,g) - V(s,g)$. Therefore, we can reduce \eqref{eq:f-divergence-dual-problem} to an unconstrained minimization problem over the dual variables
\begin{equation}
     \quad \min_{V(s,g)\geq0}  (1-\gamma)\BE_{(s,g)\sim \mu_0,p(g)}[V(s;g)] +\BE_{(s,a,g) \sim d^O}\left[f_\star \left(r(s,g) + \gamma \mathcal{T}V(s,a;g) - V(s;g)\right)\right],
\end{equation}
and consequently, we can relate the dual-optimal $V^*$ to the primal-optimal $d^*$ using Fenchel duality (see Appendix~\ref{appendix:additional-technical-background}:
\begin{equation}
    d^*(s,a;g) = d^O(s,a;g) f'_\star\left(r(s,g) + \gamma \mathcal{T}V^*(s,a,g) - V^*(s,g)\right), \forall s\in S, a\in A, g\in G,
\end{equation}
as desired. 
\end{proof}

\subsection{Proof of Theorem~\ref{theorem:gofar-bound}}
\begin{theorem}
Assume $\sup_{s,a,g} \frac{d^*(s,a;g)}{d^O(s,a;g)} \leq M$ and $\sup \lvert r(s,g) \rvert \leq R_{\max}$. Consider a policy class $\Pi:\{S\rightarrow \Delta(A)\}$ such that $\pi^* \in \Pi$. Then, for any $\delta \in (0,1]$, with probability at least $1-\delta$, GoFAR~\eqref{eq:pi-gofar} will return a policy $\hat{\pi}$ such that:
\begin{equation}
V^* - V^{\hat{\pi}} \leq \frac{2R_{\max}M}{(1-\gamma)^2} \sqrt{\frac{\ln(|\Pi|/\delta)}{N}} 
\end{equation}
where $V^\pi \coloneqq \mathbb{E}_{(s,g) \sim (\mu_0, g)}[V(s;g)].$
\end{theorem}

\begin{proof}

We begin by deriving an upper bound using the performance difference lemma~\cite{agarwal2019reinforcement}:
\begin{equation}
    V^* - V^{\hat{\pi}} \leq \frac{1}{1-\gamma} \mathbb{E}_{s \sim d^*, g\sim p(g)} \mathbb{E}_{a \sim \pi^*(\cdot \mid s,g)}A^{\hat{\pi}}(s,a;g)
\end{equation}
Then, using standard algebraic manipulations, we have:
\begin{equation}
    \begin{split}
        & \frac{1}{1-\gamma} \mathbb{E}_{s \sim d^*, g\sim p(g)} \mathbb{E}_{a \sim \pi^*(\cdot \mid s,g)}A^{\hat{\pi}}(s,a;g) \\ 
        =& \frac{1}{1-\gamma} \mathbb{E}_{s \sim d^*, g\sim p(g)} \left[\mathbb{E}_{a \sim \pi^*(\cdot \mid s,g)}A^{\hat{\pi}}(s,a;g) - \mathbb{E}_{a \sim \hat{\pi}(\cdot \mid s,g)}A^{\hat{\pi}}(s,a;g)\right] \\ 
        \leq & \frac{R_{\max}}{(1-\gamma)^2} \mathbb{E}_{s \sim d^*, g\sim p(g)} \left[\norm{\pi^*(\cdot \mid s,g) - \hat{\pi}(\cdot \mid s,g)}_1\right] \\ 
        \leq & \frac{2 R_{\max}}{(1-\gamma)^2} \mathbb{E}_{s \sim d^*, g\sim p(g)} \left[\norm{\pi^*(\cdot \mid s,g) - \hat{\pi}(\cdot \mid s,g)}_{\mathrm{TV}}\right] \\ 
    \end{split}
\end{equation}

Then, since $\sup_{s,a,g} \frac{d^*(s,a;g)}{d^O(s,a;g)} \leq M$, we can use Hoeffding's inequality with weighted empirical loss~\cite{cortes2010learning} to obtain that:
\begin{equation}
V^* - V^{\hat{\pi}} \leq \frac{2R_{\max}M}{(1-\gamma)^2} \sqrt{\frac{\ln(|\Pi|/\delta)}{N}} 
\end{equation}
\end{proof}

\section{GoFAR Technical Details}
\label{appendix:gofar-details} 
In this section, we provide additional technical details of GoFAR that are omitted in the main text. These include (1) detail of the GoFAR discriminator training, (2) mathematical expressions of GoFAR specialized to common $f$-Divergences, and (3) a full pseudocode. 
\subsection{Discriminator Training}
Training the discriminator~\ref{eq:discriminator-training} in practice requires choosing $p(s;g)$. For simplicity, we set $p(s;g)$ to be the Dirac distribution centered at $g$: $\mathbb{I}(\phi(s)=g)$; this precludes having to choose hyperparameters for $p(s;g)$. 

Once the discriminator has converged, we can retrieve the reward function $R(s;g) = \log \frac{p(s;g)}{d^O(s;g)} = - \log \left(\frac{1}{c^*(s;g)}-1\right)$, since $c^*(s;g) = \frac{d^O(s;g)}{p(s;g) + d^O(s;g)}$. 

\subsection{GoFAR with common $f$-Divergences}
GoFAR requires choosing a $f$-divergence. Here, we specialize GoFAR to $\chi^2$-divergence as well as KL-divergence as examples. Our practical implementation uses $\chi^2$-divergence, which we found to be significantly more stable than KL-divergence (see Section~\ref{section:offline-gcrl}). 

\begin{example}[GoFAR with $\chi^2$-divergence]
\label{appendix:gofar-example-chi}
$f(x) = \frac{1}{2}(x-1)^2$, and we can show that $f_\star(x) = \frac{1}{2}(x+1)^2$ and $f'_\star(x) = x+1$. Hence, the GoFAR objective amounts to
\begin{equation}
\begin{split}
&\min_{V(s;g)\geq0} (1-\gamma)\BE_{(s,g)\sim (\mu_0,p(g))}[V(s;g)] + \frac{1}{2}\BE_{(s,a,g) \sim d^O}\left[\left(R(s;g) + \gamma \mathcal{T}V(s,a;g) - V(s;g) + 1\right)^2\right]
\end{split}
\end{equation}
and 
\begin{equation}
     d^*(s,a;g) = d^O(s,a;g)\max\left(0, R(s,a;g) +\gamma \mathcal{T}V^*(s,a;g)-V^*(s;g) + 1 \right) 
\end{equation}
\end{example}

\begin{example}[GoFAR with KL-divergence]
We have $f(x) = x\log x$ and that $\D_{\star, f}(y) = \log \BE_{x\sim q}[\mathrm{exp} y(x)]$~\cite{boyd2004convex}. Hence, the KL-divergence GoFAR objective is 
\begin{equation}
\begin{split}
&\min_{V(s;g)\geq0} (1-\gamma)\BE_{(s,g)\sim (\mu_0,p(g))}[V(s;g)]+ \log \BE_{(s,a,g) \sim d^O}\left[ \mathrm{exp}\left(R(s;g) + \gamma \mathcal{T}V(s,a;g) - V(s;g)\right)\right]
\end{split}
\end{equation}
and 
\begin{equation}
     d^*(s,a;g) = d^O(s,a;g)\mathrm{softmax}\left(R(s;g)+\gamma \mathcal{T}V^*(s,a;g)-V^*(s;g)\right)
\end{equation}
\end{example}

Now, we provide the full pseudocode for GoFAR implemented using $\chi^2$-divergence in Algorithm~\ref{alg:gofar-deep}. 

\subsection{Full Pseudocode}
\vspace{0.5cm}
\begin{algorithm}[H]
\caption{GoFAR for Continuous MDPs}
\label{alg:gofar-deep}
\begin{algorithmic}[1]
\small
\STATE \textbf{Require}: Offline dataset $d^O$, choice of $f$-divergence $f$, choice of $p(s;g)$
\STATE Randomly initialize discriminator $c_\psi$, value function $V_\theta$, and policy $\pi_\phi$. 
\STATE \textcolor{purple}{\texttt{// Train Discriminator (Optional)}}
\FOR{\text{number of discriminator iterations}}
\STATE Sample minibatch $\{s_d^i, g^i\}_{i=1}^N\sim d^O$
\STATE Sample $\{s_g^j\}_{j=1}^M\sim p(s;g^i) \quad \forall i \in 1\dots N$ 
\STATE Discriminator objective: $\mc{L}_c(\psi) = \frac{1}{N}\sum_{i=1}^N[\log (1-c_\psi(s_d^j, g^i)) + \frac{1}{M} \sum_{j=1}^M [\log c_{\psi}(s_g^j, g^i)]]$
\STATE Update $c_\psi$ using SGD: $c_\psi \leftarrow c_\psi - \alpha_c \nabla \mc{L}_c(\psi)$
\ENDFOR
\STATE \textcolor{purple}{\texttt{// Train Dual Value Function}}
\FOR{\text{number of value iterations}}
\STATE Sample minibatch of offline data $\{s_t^i, a_t^i, s_{t+1}^i, g_t^i\}_{i=1}^N \sim d^O, \{s_0^i\}_{i=1}^M\sim \mu_0$, $\{g_0^i\}_{i=1}^M \sim d^O$
\STATE If discriminator, obtain reward: $R(s_t^i;g_t^i) = - \log \left(\frac{1}{c_\psi(s_t^i;g_t^i)}-1\right) \quad \forall i = 1\dots N$
\STATE If no discriminator, obtain reward: $\{R(s_t^i;g_t^i)\}_{i=1}^N \sim d^{O}$
\STATE Value objective: $\mc{L}_V(\theta) = \frac{1-\gamma}{M}\sum_{i=1}^M[V_\theta(s_0^i ; g_0^i)]+\frac{1}{N}\sum_{i=1}^N\left[f_{\star}(R_t^i+\gamma V(s_{t+1}^i; g_t^i)-V(s_t^i; g_t^i))\right]$
\STATE Update $V_\theta$ using SGD: $V_\theta \leftarrow V_\theta - \alpha_V
\nabla \mc{L}_V(\theta)$
\ENDFOR

\STATE \textcolor{purple}{\texttt{// Train Policy With $f$-Advantage Regression}}
\FOR{\text{number of policy iterations}}
\STATE Sample minibatch of offline data $\{s_t^i, a_t^i, s_{t+1}^i, g_t^i\}_{i=1}^N \sim d^O$
\STATE If discriminator, obtain reward: $R(s_t^i;g_t^i) = - \log \left(\frac{1}{c_\psi(s_t^i;g_t^i)}-1\right) \quad \forall i = 1\dots N$
\STATE If no discriminator, obtain reward: $\{R(s_t^i;g_t^i)\}_{i=1}^N \sim d^{O}$
\STATE Policy objective: $\mc{L}_\pi(\phi) = \sum_{i=1}^N\left[\left(f_{\star}^{\prime}\left(R_t^i+\gamma V_\theta(s_{t+1}^i; g_t^i)-V_\theta(s_t^i; g_t^i\right)\right) \log \pi(a \mid s, g)\right]$
\STATE Update $\pi_\phi$ using SGD: $\pi_\phi \leftarrow \pi_\phi - \alpha_\pi
\nabla \mc{L}(\phi)$
\ENDFOR
\end{algorithmic}
\end{algorithm}

\section{GoFAR for Tabular MDPs}
\label{appendix:gofar-tabular-mdp} 
In Section 4.1 of the main text, we have stated that in tabular MDPs,  GoFAR's optimal dual value function~\eqref{eq:V-problem} admits closed-form solution when we choose $\chi^2$-divergence. Here, we provide a derivation of this result.

Recall the dual problem~\eqref{appendix:gofar-example-chi}
\begin{equation}
\min_{V(s;g)\geq0} (1-\gamma)\BE_{(s,g)\sim (\mu_0,p(g))}[V(s;g)] + \frac{1}{2}\BE_{(s,a,g) \sim d^O}\left[\left(R(s;g) + \gamma \mathcal{T}V(s,a;g) - V(s;g) + 1\right)^2\right]
\end{equation}

To derive a closed-form solution, we rewrite the problem in vectorized notation; we borrow our notations from~\cite{ma2022smodice}. We first augment the state-space by concatenating the state dimensions and the goal dimensions so that the new state space $\tilde{S}$ has dimension $S+G$. Then, the new transition function, with slight abuse of notation, $T((s',g')\mid (s,g),a) = T(s\mid s,a) \mathbb{I}(g'=g)$; the new initial state distribution is thus $\mu_0(s,g) = \mu_0(s)p(g)$. Therefore, $\tilde{T} \in \mathbb{R}_+^{|S||G||A| \times |S||G|}$ and $\mu_0 \in \mathbb{R}_+^{|S||G|}$. 

We assume that the offline dataset $\mathcal{D}^O$ is collected by a behavior policy $\pi_b$. We construct a surrogate MDP $\hat{\mathcal{M}}$ using maximum likelihood estimation; that is, $\hat{T}((s',g')\mid (s,g),a) = \frac{n(s,a,s')}{n(s,a)} \mathbb{I}(g'=g)$, and we impute $\hat{T}((s',g)\mid (s,g),a) = \frac{1}{S}$ when $n(s,a) = 0$. Then, using $\hat{M}$, we can compute $d^O\in \mathbb{R}_+^{|S||G||A|}$ using linear programming and define reward $r \in \mathbb{R}_+^{|S||G|}$ as $r(s;g) = \log \frac{p(s;g)}{d^O(s;g)}$, where $p(s;g) \in \mathbb{R}_+^{|S||G|}$. Now, define $\mathcal{T} \in \mathbb{R}^{|S||G||A|\times |S||G|}$ such that $(\mathcal{T}V)(s,a;g) = \sum{s'} T((s',g) \mid (s,g), a) V(s';g)$, where $V\in \mathbb{R}_+^{|S||G|}$ is the dual optimization variables. We also define $\mathcal{B} \in \mathbb{R}^{|S||G||A| \times |S||G|}$ such that $(\mathcal{B}V)(s,a;g) = V(s;g)$. Finally, we define $D = \mathrm{diag}(d^O) \in \mathbb{R}^{|S||G||A| \times |S||G||A|}$. Now, we can rewrite the dual problem as follow:
\begin{equation}
\begin{split}
&\min_{V(s;g)\geq0} (1-\gamma)\BE_{(s,g)\sim (\mu_0,p(g))}[V(s;g)] + \frac{1}{2}\BE_{(s,a,g) \sim d^O}\left[\left(R(s;g) + \gamma \mathcal{T}V(s,a;g) - V(s;g) + 1\right)^2\right] \\
   \Rightarrow &\min_{V(s;g)} (1-\gamma) \mu_0^{\top} V +\frac{1}{2}\BE_{(s,a,G) \sim d^O}\left[\left(\underbrace{\mathcal{B}R(s,a;g) + \gamma \mathcal{T}V(s,a;g) - \mathcal{B}V(s,a;g)}_{R_V(s,a;g)} + 1\right)^2\right] \\ 
   \label{eq:gofar-tabular-last}
   \Rightarrow &\min_{V(s;g)} (1-\gamma) \mu_0^{\top} V + \frac{1}{2} (R_V + I)^{\top} D (R_V+ I)
\end{split}
\end{equation}

Now, we recognize that~\eqref{eq:gofar-tabular-last} is equivalent to Equation 49 in~\cite{ma2022smodice}, as we have reduced goal-conditioned RL to regular RL with an augmented state-space. Now, using the same derivation as in~\cite{ma2022smodice}, we have that
\begin{equation}
    \label{eq:gofar-chi-tabular-optimal-V}
    V^* = \left((\gamma \mc{T} - \mc{B})^\top D (\gamma \mc{T} - \mc{B})\right)^{-1}\left((\gamma-1)\mu_0 + (\mc{B} - \gamma \mc{T})^\top D(I+BR) \right)
\end{equation}
and we can recover $d^*(s,a;g)$:
\begin{equation}
    \label{eq:gofar-chi-tabular-optimal-ratio}
    d^*(s,a;g)= d^O(s,a;g)\left(\mc{B}R(s,a;g) + \gamma \mc{T}V^*(s,a;g) - \mc{B}V^*(s,a;g) + 1\right)
\end{equation}

Given $d^*$, we may extract the optimal policy $\pi^*$ by marginalizing over actions:
\begin{equation}
    \label{eq:pi-star}
    \pi^*(a\mid s, g) = \frac{d^*(s,a;g)}{\sum_a d^*(s,a';g)} = \frac{d^O(s,a;g) \left(R(s;g) + \gamma \mathcal{T}V(s,a;g) - V(s;g)\right)}{\sum_a d^O(s,a;g) \left(R(s;g) + \gamma \mathcal{T}V(s,a;g) - V(s;g)\right)}
\end{equation}

\section{Additional Technical Discussion}
\label{appendix:technical-discussion} 

\subsection{Connecting Goal-Conditioned State-Matching and Probabilistic GCRL}
\label{appendix:probabilistic-gcrl}

Suppose the GCRL problem comes with a reward function $r(s;g)$. We also show that there is an equivalent goal-conditioned state-occupancy matching problem with a target distribution $p(s;g)$ defined based on $r(s;g)$.

\begin{proposition}
(Proposition~\ref{proposition:gcrl-objective} in the paper)
Given any $r(s;g)$, for each $g$ in the support of $p(g)$, define $p(s;g) = \frac{e^{r(s;g)}}{Z(g)}$, where $Z(g) := \int e^{r(s;g)} ds$ is the normalizing constant. Then, the following equality holds: 
\begin{equation}
    -\D_\KL(d^\pi(s;g) \| p(s;g)) + C = (1-\gamma) J(\pi) + \mathcal{H}(d^\pi(s;g)) 
\end{equation}
where $J(\pi)$ is the GCRL objective (Eq.~\eqref{eq:gcrl-objective}) with reward $r(s;g)$ and $C := \mathbb{E}_{g \sim p(g)}[\log Z(g)]$ is a constant. 
\end{proposition}
\begin{proof}
We have that 
\begin{equation}
\begin{split}
& (1-\gamma)J(\pi) \\ 
    =& \mathbb{E}_{g \sim p(g)} \mathbb{E}_{s \sim d^\pi(s;g)}\left[r(s;g)\right] \\ 
=& \mathbb{E}_{g \sim p(g)} \mathbb{E}_{s \sim d^\pi(s;g)}\left[\log e^{r(s;g)} \right] \\
        =& \mathbb{E}_{g \sim p(g)} \mathbb{E}_{s \sim d^\pi(s;g)}\left[\log \frac{e^{r(s;g)}Z(g)}{Z(g)}\right] \\
=& \mathbb{E}_{g \sim p(g)} \mathbb{E}_{s \sim d^\pi(s;g)}\left[\log \frac{e^{r(s;g)}}{Z(g)}\right] + \mathbb{E}_{g\sim p(g)}[\log Z(g)] \\ 
=&\mathbb{E}_{g \sim p(g)} \mathbb{E}_{s \sim d^\pi(s;g)}\left[\log \frac{e^{r(s;g)}}{Z(g)}\cdot \frac{d^\pi(s;g)}{d^\pi(s;g)} \right] + C \\ 
=& \mathbb{E}_{g \sim p(g)} \mathbb{E}_{s \sim d^\pi(s;g)}\left[\log \frac{p(s;g)}{d^\pi(s;g)} \right] + \mathbb{E}_{g \sim p(g)}\mathbb{E}_{d^\pi(s;g)}[\log d^\pi(s;g)]+ C \\ 
        =& \mathbb{E}_{g\sim p(g)} \left[-\D_\KL(d^\pi(s;g) \| p(s;g)) - \mathcal{H}(d^\pi(s;g)) \right]  + C
\end{split}
\end{equation}
Rearranging the inequality gives the desired result.
\end{proof}

A constant term $C$ appear in the equality to account for the need for normalizing $e^{r(s;g)}$ to make it a proper distribution. This, however, does not change the optimal solution for the goal-conditioned state-occupancy matching objective. Therefore, we have shown that for any choice of reward $r(s;g)$, solving the GCRL problem with a maximum state-entropy regularization is equivalent to optimizing for the goal-conditioned state-occupancy matching objective with target distribution $p(s;g) := \frac{e^{r(s;g)}}{Z(g)}$.

\subsection{Optimality Conditions for Hindsight Relabeling} 
\label{appendix:her-optimality}
In section 4.2, we have stated that HER is not optimal for most choices of reward functions. In this section, we investigate conditions under which hindsight relabeling methods such as HER would be optimal. 

Let the goal-relabeling distribution for HER be $p_{\mathrm{HER}}(g\mid s,a)$; we do not specify the functional form of $p_{\mathrm{HER}}(g\mid s,a)$ for generality (see~\ref{eq:her-definition}). Then, in order for this distribution to be optimal, then it must satisfy
\begin{equation}
    p_{\mathrm{HER}}(g \mid s,a) = p(g\mid s,a) (f'_\star (R(s;g) + \gamma \mathcal{T}V^*(s,a;g) - V^*(s;g))
\end{equation}
Then, the choice of $r(s;g)$ such that this equality holds is the reward function for which HER would be optimal. However, solving for $r(s;g)$ is generally challenging and we leave it to future work for investigating whether doing so is possible for general $f$-divergence coupled with neural networks. 

This optimality condition is related to a prior work~\cite{eysenbach2020rewriting}, which has found that hindsight relabeling is optimal in the sense of maximum-entropy inverse RL~\cite{ziebart2008maximum} for a peculiar choice of reward function (see Equation 9 in~\cite{eysenbach2020rewriting}), which cannot be implemented in practice. Our result is more general as it applies to any choice of $f$-divergence, and is not restricted to the form of maximum-entropy inverse RL.

\subsection{Theoretical Comparison to Prior Regression-based GCRL methods}
\label{appendix:theoretical-comparison} 
In section 4.3, we have stated that GoFAR's theoretical guarantee (Theorem~\ref{theorem:gofar-bound}) is stronger in nature compared to prior regression-based GCRL methods. Here, we provide an in-depth discussion.

Both GCSL~\cite{ghosh2021learning} and WGCSL~\cite{yang2022rethinking} prove that their objectives are lower bounds of the true RL objective (Theorem 3.1 in~\cite{ghosh2021learning} and Theorem 1 in~\cite{yang2022rethinking}, respectively); however, in both works, the lower bounds are loose due to constant terms that do not depend on the policy and hence do not vanish to zero. In contrast, GoFAR's objective~\eqref{eq:f-divergence-bound} is, by construction, a lower bound on the RL objective, as it simply incorporates a $f$-divergence regularization. If the offline data $d^O$ is \textit{on-policy}, then our lower bound is an equality. In contrast, even with on-policy data, the lower bounds in both GCSL and WGCSL are still loose due to the unavoidable constant terms.

GCSL also proves a sub-optimality guarantee (Theorem 3.2 in~\cite{ghosh2021learning}) under the assumption of full state-space coverage. Though full state-space coverage has been considered in some prior offline RL works~\cite{kumar2020conservative, ma2021conservative}, it is much stronger than the concentrability assumption in our Theorem~\ref{theorem:gofar-bound}, which only applies to $d^*$. Furthermore, this guarantee is not statistical in nature, and instead directly makes a strong assumption on the \textit{maximum} total-variance distance between $\pi$ and optimal $\pi^*$ for the GCSL objective, which is difficult to verify in practice. In contrast, our bound suggests asymptotic optimality: given enough offline data, the solution to GoFAR's policy objective will converge to $\pi^*$. Finally, WGCSL proves a policy improvement guarantee (Proposition 1 in~\cite{yang2022rethinking}) under their exponentially weighted advantage; the improvement is not a strict equality, and consequently there is no convergence guarantee to the optimal policy. Furthermore, this result is not directly dependent on their use of an advantage function, so it is not clear the precise role of their advantage function in their algorithm.

\section{Task Descriptions} 
In this section, we describe the tasks in our experiments in Section 5. 

\label{appendix:task-descriptions}
\begin{figure}[H]
\centering
\includegraphics[width=\linewidth]{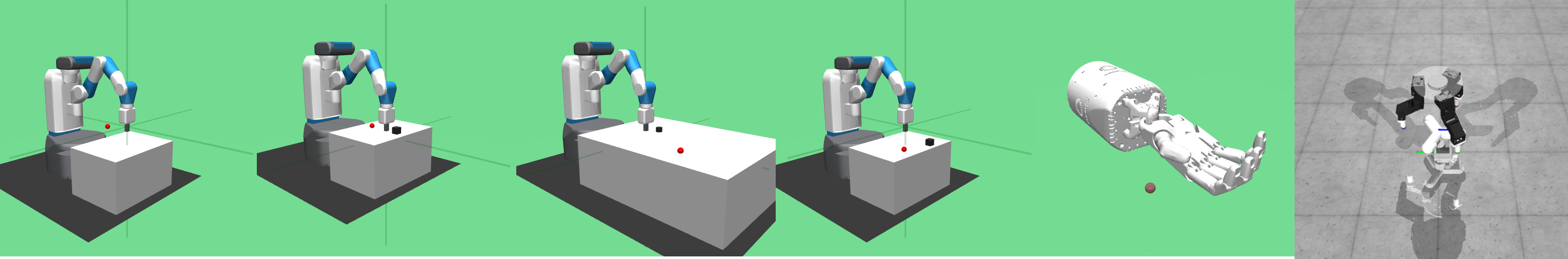}
\caption{Tasks from left to right: FetchReach, FetchPush, FetchSlide, FetchPick, HandReach, D'Claw (Simulation)}
\label{figure:task-descr}
\end{figure}

\subsection{Fetch Tasks}

The Fetch environments involve the Fetch robot with the following specifications, developed by Plappert et al. \cite{plappert2018multi}. 
\begin{itemize}
    \item Seven degrees of freedom
    \item Two pronged parallel gripper
    \item Three dimensional goal representing Cartesian coordinates of target
    \item Sparse, binary reward signal: 0 when at goal with tolerance 5cm and -1 otherwisejk
    \item 25 Hertz simulation frequency
    \item Four dimensional action space
        \begin{itemize}
            \item Three Cartesian dimensions
            \item One dimension to control gripper
        \end{itemize}
\end{itemize}
\subsubsection{Fetch Reach}
The task is to place the end effector at the target goal position. Observations consist of the end effector's positional state, whether the gripper is closed, and the end effector's velocity. The reward is given by:
\[r(s, a, g) = 1 - \mathbbm{1}(\norm{s_{xyz, \text{ee}} - g_{xyz}}_2 \leq 0.05)\]

\subsubsection{Fetch Push}
The task is to push an object to the target goal position. Observations consist of the end effector's position, velocity, and gripper state as well as the object's position, rotational orientation, linear velocity, and angular velocity. The reward is given by:
\[r(s, a, g) = 1 - \mathbbm{1}(\norm{s_{xyz, \text{obj}} - g_{xyz}}_2 \leq 0.05)\]

\subsubsection{Fetch Slide}
In this task, the goal position lies outside of the robot's reach and the robot must slide the puck-like object across the table to the goal.  Observations consist of the end effector's position, velocity, and gripper state as well as the object's position, rotational orientation, linear velocity, and angular velocity. The reward is given by:
\[r(s, a, g) = 1 - \mathbbm{1}(\norm{s_{xyz, \text{obj}} - g_{xyz}}_2 \leq 0.05)\]

\subsubsection{Fetch Pick}
The task is to grasp the object and hold it at the goal, which could be on or above the table. Observations consist of the end effector's position, velocity, and gripper state as well as the object's position, rotational orientation, linear velocity, and angular velocity. The reward is given by:
\[r(s, a, g) = 1 - \mathbbm{1}(\norm{s_{xyz, \text{obj}} - g_{xyz}}_2 \leq 0.05)\]

\subsection{Hand Reach}
Uses a 24 DoF robot hand with a 20 dimensional action space. Observations consist of each of the 24 joints' positions and velocities. The goal space is 15 dimensional corresponding to the positions of each of its five fingers. The goal is achieved when the mean distance of the fingers to their goals is less than 1cm. The reward is binary and sparse: 0 if the goal is reached and -1 otherwise, i.e.
\[r(s, a, g) = 1 - \mathbbm{1}\left(\frac{1}{5}\sum_{i=1}^5 \norm{s_{i} - g_i}_2 \leq 0.01\right)\]

\subsection{D'ClawTurn (Simulation)}
First introduced by Ahn et al. \cite{ahn2020robel}, the D'Claw environment has a 9 DoF three-fingered robotic hand. The turn task consists of turning the valve to a desired angle. The initial angle is randomly chosen from $[-\frac{\pi}{3}, \frac{\pi}{3}]$; the target angle is randomly chosen from $[-\frac{2*\pi}{3}, \frac{2*\pi}{3}]$. The observation space is 21D, consisting of the current joint angles $\theta_t$, their velocities $\dot{\boldsymbol{\theta}}$, angle between current and goal angle, and the previous action. The environment terminates after 80 steps. The reward function is defined as:
\[r = \mathbbm{1}\left(\left|\operatorname{arctan2}\left(\frac{s_{y,obj}}{s_{x,obj}}\right) - \operatorname{arctan2}\left(\frac{g_{y,obj}}{g_{x,obj}}\right)\right| \leq 0.1\right)\]

\subsection{D'ClawTurn (Real)}
To make real-world data collection easier, we slightly modify the initial and target angle distributions. The initial angle is randomly chosen from $[-\frac{\pi}{3}, \frac{\pi}{3}]$; the target angle is randomly chosen from $[-\frac{\pi}{2}, \frac{\pi}{2}]$. Using this task distribution, collecting 400K transitions with random actions takes about 15 hours. In Figure~\ref{figure:dclaw-large}, we also include a larger picture of the robot platform.
\vspace{0.2cm}
\begin{figure}[H]
    \centering
    \includegraphics[width=0.4\linewidth]{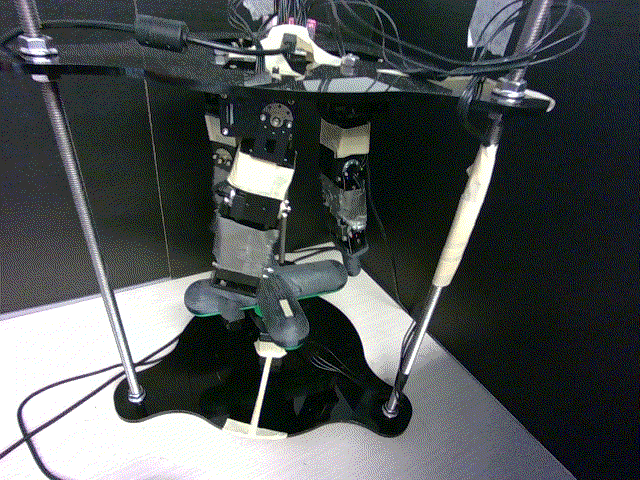}
    \caption{The D'Claw tri-finger platform.}
    \label{figure:dclaw-large}
\end{figure}

\section{Experimental Details}
\label{appendix:experimental-details}
In this section, we provide experimental details omitted in Section 5 of the main text. These include (1) technical details of the baseline methods, (2) hyperparameter and architecture details for all methods, (3) offline GCRL dataset details, and finally, (4) experimental details of the zero-shot transfer experiment. 

\subsection{Baseline Implementation Details} 
\para{DDPG.} We use an open-source implementation of DDPG, which has already tuned DDPG on the set of Fetch tasks. We implement all other methods on top of this implementation, keeping identical architectures and hyperparameters when appropriate. The critic objective is
\begin{equation}
    \label{eq:ddpg-critic-loss}
    \min_Q \mathbb{E}_{(s_t,a_t,s_{t+1},g) \sim d^{\tilde{O}}}[(r(s_t,g) + \gamma \bar{Q}(s_{t+1}, \pi(s_{t+1},g), g) - Q(s_t,a_t,g))^2]
\end{equation}
where $\bar{Q}$ denotes the stop-gradient operation. The policy objective is
\begin{equation}
    \label{eq:ddpg-actor-loss}
    \min_\pi - \mathbb{E}_{(s_t,a_t,s_{t+1},g) \sim d^{\tilde{O}}}[Q(s_t, \pi(s_t, g), g)]
\end{equation}
DDPG updates the critic and the policy in an alternating fashion. 

\para{ActionableModel.} We implement AM on top of DDPG. Specifically, we add a CQL loss in the critic update:
\begin{equation}
    \label{eq:cql-loss}
    \mathbb{E}_{(s,g)\sim d^{\tilde{O}}, a\sim \mathrm{exp}(Q)}[Q(s,a,g)]
\end{equation}
where $d^{\tilde{O}}$ is the distribution of the relabelled dataset. In practice, we sample $10$ random actions from the action-space to approximate this expectation. Furthermore, we implement goal-chaining, where for half of the relabeled transitions in each minibatch update, the relabelled goals are randomly sampled from the offline dataset. We found goal-chaining to not be stable in some environments, in particular, FetchPush, FetchPickAndPlace, and FetchSlide. Therefore, to obtain better results, we remove goal-chaining for these environments in our experiments. 

\para{GCSL.} We implement GCSL by removing the DDPG critic component and changing the policy loss to maximum likelihood:
\begin{equation}
    \label{eq:gcsl-loss}
    \min_\pi -\mathbb{E}_{(s,a, g)\sim d^{\tilde{O}}}\left[ \log \pi(a \mid s, g)\right]
\end{equation}

\para{WGCSL.} We implement WGCSL on top of GCSL by including a Q-function. The Q-function is trained using TD error as in DDPG and provided an advantage weighting in the regression loss. The advantage term we compute is $A(s_t,a_t,g) = r(s_t;g) + \gamma Q(s_{t+1}, \pi(s_{t+1},g);g)- Q(s_t,a_t;g)$.
Using this, the WGCSL policy objective is
\begin{equation}
\label{eq:wgcsl-loss}
\min_\pi -\mathbb{E}_{(s_t,a_t,\phi(s_i))\sim d^{\tilde{O}}}\left[\gamma^{i-t}\mathrm{exp}_{clip}(A(s_t,a_t,\phi(s_i)))\log \pi(a_t \mid s_t, \phi(s_i))\right]
\end{equation}
where we clip $\mathrm{exp}(\cdot)$ for numerical stability. The original WGCSL uses different HER rates for the critic and the actor training. To make the implementation simple and consistent with all other approaches, we use the same HER rate for both components. We note that the original WGCSL computes the advantage term slightly differently as $A(s_t,a_t,g) = r(s_t;g) + \gamma Q(s_{t+1}, \pi(s_{t+1},g);g) - Q(s_t, \pi(s_{t},g);g)$; this version of WGCSL\footnote{We thank Joey Hejna for pointing out this difference in an email correspondence.} is incorporated in our open-sourced code.

With the exception of AM, all baselines set the goal-relabeling distribution $d^{\tilde{O}}$ to be the uniform distribution over future states in the same trajectory (See Equation~\eqref{eq:her-definition}).

\subsection{Architectures and Hyperparameters}
Each algorithm uses their own set of fixed hyperparameters for all tasks. WGCSL, GCSL, and DDPG are already tuned on our set of tasks~\cite{plappert2018multi, yang2022rethinking}, so we use the reported values from prior works; AM, in our implementation, shares same networks as DDPG, so we use DDPG's values. For GoFAR, we use identical hyperparameters as WGCSL because they share similar network components; GoFAR additionally trains a discriminator, for which we use the same architecture and learning rate as the value network. We impose a small discriminator gradient penalty~\cite{gulrajani2017improved} to prevent overfitting. For all experiments,
We train each method for 3 seeds, and each training run uses 400k minibatch updates of size 512.
The architectures and hyperparameters for all methods are reported in Table~\ref{table:gofar-hyperparameters}.

\begin{table}[ht]
\centering
\caption{Offline GCRL Hyperparameters.}\label{table:gofar-hyperparameters}
\begin{tabular}{cll}
\toprule
& Hyperparameter & Value \\
\midrule
Hyperparameters & Optimizer & Adam~\cite{kingma2014adam} \\
                                      & Critic learning rate & 5e-4 (1e-3 for AM/DDPG) \\
                                      & Actor learning rate  & 5e-4 (1e-3 for AM/DDPG) \\
                                      & Discriminator learning rate & 5e-4 \\
                                      & Discriminator gradient penalty & 0.01 \\
                                      & Mini-batch size      & 256 \\
                                      & Discount factor      & 0.98 \\
                                      
\midrule
Architecture    
                                      & Discriminator hidden dim     & 256        \\
                                      & Discriminator hidden layers  & 2          \\
                                      & Discriminator activation function & ReLU \\
                                      & Critic (resp. Value) hidden dim    & 256        \\
                                      & Critic (resp. Value) hidden layers & 2          \\
                                      & Critic (resp.Value) activation function & ReLU \\
                                      & Actor hidden dim     & 256        \\
                                      & Actor hidden layers  & 2          \\
                                      & Actor activation function & ReLU \\
\bottomrule
\end{tabular}
\end{table} %

\subsection{Offline GCRL Experiments}
\label{appendix:offline-gcrl} 

\para{Datasets.} For each environment, the offline dataset composition is determined by whether data collected by random actions provides sufficient coverage of the desired goal distribution. For FetchReach and D'ClawTurn, we find this to be the case and choose the offline dataset to be 1 million random transitions. For the other four tasks, random data does not capture meaningful goals, so we create a mixture dataset with 100K transitions from a trained DDPG-HER agent and 900K random transitions; the transitions are not labeled with their sources. This mixture setup has been considered in prior works~\cite{kim2022demodice, ma2022smodice} and is reminiscent of real-world datasets, where only a small portion of the dataset is task-relevant but all transitions provide useful information about the environment. 

\subsection{Zero-Shot Transfer Experiments}
\label{appendix:zero-shot-transfer} 
We use GoFAR (Binary) variant for trainning the GoFAR planner. The low-level controller is trained using an online DDPG algorithm on a narrow goal distribution, set to be closed to the object's initial positions. 

GoFAR Hierarchical Controller operates by first generating a sequence of subgoals $(g_1,...,g_T)$ using $\pi_{\mathrm{high}}$ by recursively feeding the newest generated goal and conditioning on the final goal $g$. Then, at each time step $t$, the low-level controller executes action $\pi_{\mathrm{low}}(a_t \mid s_t, g_t)$. The high-level subgoals are not re-planned during low-level controller execution. We note that this is a simple planning algorithm, and improvement in performance can be expected by considering more sophisticated planning approaches.
\section{Additional Results}
\label{appendix:additional-results} 

\subsection{Offline GCRL Full Results}
In this section, we provide the full results table for discounted return, final distance, and success rate metrics, including error bars over 10 random seeds. The number inside the parenthesis indicates the best HER rate for the baseline methods on the task. Star ($\star$) indicate statistically significant improvement over the second best-performing method under a 2-sample $t$-test.

\begin{table}[H]
\caption{Discounted Return on offline GCRL tasks, averaged over $10$ random seeds.}
\label{table:offline-gcrl-full-discounted-return-appendix}
\centering
\resizebox{\textwidth}{!}{
\begin{tabular}{l|rrr|rr}
\toprule
\multicolumn{1}{c|}{\textbf{Task}}
& \multicolumn{3}{c|}{\textbf{Supervised Learning}}
& \multicolumn{2}{c}{\textbf{Actor-Critic}} \\ 
 &  \textbf{GoFAR} (Ours) & \textbf{WGCSL} & \textbf{GCSL} & \textbf{AM} &\textbf{DDPG} \\ 
\midrule 
FetchReach  & 28.2 $\pm$ {\scriptsize0.61} & 21.9$\pm$ {\scriptsize2.13} (1.0) & 20.91 $\pm$ {\scriptsize2.78} (1.0)  & \textbf{30.1} $\pm$ {\scriptsize0.32} (0.5) & 29.8 $\pm$ {\scriptsize0.59} (0.2) \\ 
FetchPick & \textbf{19.7} $\pm$ {\scriptsize2.57} & 9.84 $\pm$ {\scriptsize2.58} (1.0) & 8.94 $\pm$ {\scriptsize3.09} (1.0) & 18.4 $\pm$ {\scriptsize3.51} (0.5) & 16.8 $\pm$ {\scriptsize3.10} (0.5) \\
FetchPush ($\star$) & \textbf{18.2} $\pm$ {\scriptsize3.00} &14.7 $\pm$ {\scriptsize2.65} (1.0) & 13.4 $\pm$ {\scriptsize3.02} (1.0)  & 14.0 $\pm$ {\scriptsize2.81} (0.5) & 12.5 $\pm$ {\scriptsize4.93} (0.5) \\
FetchSlide & 2.47 $\pm$ {\scriptsize1.44} & \textbf{2.73} $\pm$ {\scriptsize1.64} (1.0) & 1.75 $\pm$ {\scriptsize1.3}(1.0) & 1.46 $\pm$ {\scriptsize1.38} (0.5) & 1.08 $\pm$ {\scriptsize1.35} (0.5)  \\ 
\midrule 
HandReach ($\star$) & \textbf{11.5} $\pm$ {\scriptsize5.26} & 5.97 $\pm$ {\scriptsize4.81} (1.0) & 1.37 $\pm$ {\scriptsize2.21} (1.0)& 0. $\pm$ {\scriptsize0.0} (0.5) & 0.81 $\pm$ {\scriptsize1.73} (0.5)\\ 
D'ClawTurn ($\star$) & \textbf{9.34} $\pm$ {\scriptsize 3.15} & 0.0 $\pm$ {\scriptsize0.0} (1.0) & 0.0 $\pm$ {\scriptsize0.0} (1.0) & 2.82$\pm$ {\scriptsize 1.71} (1.0) & 0.0$\pm$ {\scriptsize 0.0} (0.2) \\ 
\midrule
Average Rank &  \textbf{1.5} & 3 & 4.17 & 2.83 & 4 \\ 
\bottomrule
\end{tabular}}
\end{table}

\begin{table}[H]
\caption{Final Distance on offline GCRL tasks, averaged over $10$ random seeds.}
\label{table:offline-gcrl-full-final-distance-appendix}
\centering
\resizebox{\textwidth}{!}{
\begin{tabular}{l|rrr|rr}
\toprule
\multicolumn{1}{c|}{\textbf{Task}}
& \multicolumn{3}{c|}{\textbf{Supervised Learning}}
& \multicolumn{2}{c}{\textbf{Actor-Critic}} \\ 
 &  \textbf{GoFAR} (Ours) & \textbf{WGCSL} & \textbf{GCSL} & \textbf{AM} &  \textbf{DDPG} \\ 
\midrule 
FetchReach  & 0.018 $\pm$ {\scriptsize0.003} & 0.007 $\pm$ {\scriptsize0.0043}(1.0) & 0.008 $\pm$ {\scriptsize0.008}(1.0) & \textbf{0.007}$\pm$ \scriptsize{0.001} (0.5) & 0.041 $\pm$ {\scriptsize0.005} (0.2) \\ 
FetchPickAndPlace & \textbf{0.036} $\pm$ {\scriptsize0.013} & 0.094 $\pm$ {\scriptsize0.043}(1.0) & 0.108 $\pm$ {\scriptsize0.060}(1.0) & 0.040 $\pm$ {\scriptsize0.020}(0.5) & 0.043 $\pm$ {\scriptsize0.021}(0.5) \\
FetchPush & \textbf{0.033} $\pm$ {\scriptsize0.008} & 0.041 $\pm$ {\scriptsize0.020}(1.0) & 0.042 $\pm$ {\scriptsize0.018} (1.0) & 0.070 $\pm$ {\scriptsize0.039}(0.5) & 0.060 $\pm$ {\scriptsize0.026} (0.5) \\
FetchSlide ($\star$) & \textbf{0.120} $\pm$ {\scriptsize0.02}& 0.173 $\pm$ {\scriptsize0.04}(1.0) & 0.204 $\pm$ {\scriptsize0.051} (1.0) & 0.198 $\pm$ {\scriptsize0.059} (0.5) & 0.353 $\pm$ {\scriptsize0.248} (0.5) \\ 
\midrule 
HandReach ($\star$) & \textbf{0.024} $\pm$ {\scriptsize0.009} & 0.035 $\pm$ {\scriptsize0.012} (1.0) & 0.038 $\pm$ {\scriptsize0.013}(1.0) & 0.037 $\pm$ {\scriptsize0.004}(0.5) & 0.038 $\pm$ {\scriptsize0.013} (0.5) \\ 
D'ClawTurn ($\star$) & \textbf{0.92} $\pm$ {\scriptsize0.28} & 1.49 $\pm$ {\scriptsize0.26} (1.0) & 1.54 $\pm$ {\scriptsize0.15} (1.0) & 1.28 $\pm$ {\scriptsize0.26} (1.0) & 1.54 $\pm$ {\scriptsize0.13} (0.2) \\ 
\midrule
Average Rank & \textbf{1.5} & 2.33 & 4.25 & 2.67 & 4.5\\ 
\bottomrule
\end{tabular}}
\end{table}

\begin{table}[H]
\caption{Success Rate on offline GCRL tasks, averaged over $10$ random seeds.}
\label{table:offline-gcrl-full-final-distance-appendix}
\centering
\resizebox{\textwidth}{!}{
\begin{tabular}{l|rrr|rr}
\toprule
\multicolumn{1}{c|}{\textbf{Task}}
& \multicolumn{3}{c|}{\textbf{Supervised Learning}}
& \multicolumn{2}{c}{\textbf{Actor-Critic}} \\ 
 &  \textbf{GoFAR} (Ours) & \textbf{WGCSL} & \textbf{GCSL} & \textbf{AM} &  \textbf{DDPG} \\ 
\midrule 
FetchReach  & \textbf{1.0} $\pm$ {\scriptsize0.0} & 0.99 $\pm$ {\scriptsize 0.01}  (1.0) & 0.98 $\pm$ {\scriptsize0.05} (1.0) & \textbf{1.0} $\pm$ {\scriptsize0.0} (0.5) & 0.99 $\pm$ {\scriptsize0.02} (0.2) \\ 
FetchPickAndPlace & \textbf{0.84} $\pm$ {\scriptsize0.12} & 0.54 $\pm$ {\scriptsize0.16} (1.0) & 0.54 $\pm$ {\scriptsize0.20}(1.0) & 0.78 $\pm$ {\scriptsize0.15}(0.5) & 0.81 $\pm$ {\scriptsize0.13}(0.5) \\
FetchPush ($\star$) & \textbf{0.88} $\pm$ {\scriptsize0.09} & 0.76 $\pm$ {\scriptsize0.12}(1.0) & 0.72 $\pm$ {\scriptsize0.15}(1.0) & 0.67 $\pm$ {\scriptsize0.14} (0.5) & 0.65 $\pm$ {\scriptsize0.18} (0.5) \\
FetchSlide & \textbf{0.18} $\pm$ {\scriptsize0.12}& \textbf{0.18} $\pm$ {\scriptsize0.14}(1.0) & 0.17 $\pm$ {\scriptsize0.13} (1.0) & 0.11 $\pm$ {\scriptsize0.09} (0.5) & 0.08 $\pm$ {\scriptsize0.11} (0.5) \\ 
\midrule 
HandReach & \textbf{0.40} $\pm$ {\scriptsize0.20} & 0.25 $\pm$ {\scriptsize0.23} (1.0) & 0.047 $\pm$ {\scriptsize0.10} (1.0) & 0.0 $\pm$ {\scriptsize0.0} (0.5) &  0.023 $\pm$ {\scriptsize0.054} (0.5) \\ 
D'ClawTurn ($\star$) & \textbf{0.26} $\pm$ {\scriptsize0.13} & 0.0 $\pm$ {\scriptsize0.0} (1.0) & 0.0 $\pm$ {\scriptsize0.0} (1.0) & 0.13 $\pm$ {\scriptsize0.14} (1.0) & 0.01 $\pm$ {\scriptsize0.02} (0.2) \\ 
\midrule
Average Rank & \textbf{1} & 3 & 4 & 3.33 & 3.67 \\ 
\bottomrule
\end{tabular}}
\end{table}

\subsection{Ablations}
We also include the full task-breakdown table of GoFAR ablations presented in Figure~\ref{figure:ablation-studies} for completeness. 
As shown in~\ref{table:ablation-studies-discounted-return}, GoFAR and GoFAR (HER) perform comparatively on all tasks. GoFAR (binary) is slightly worse across tasks, and GoFAR (KL) collapses due to the use of an unstable $f$-divergence. 
\begin{table}[H]
\caption{GoFAR Ablation Studies}
\label{table:ablation-studies-discounted-return}
\centering
\resizebox{\textwidth}{!}{
\begin{tabular}{l|rrrrrr}
\toprule
\textbf{Variants} & FetchReach & FetchPickAndPlace & FetchPush & FetchSlide & HandReach & DClawTurn\\ 
\midrule 
GoFAR & 27.8 $\pm$ 0.55& 19.5 $\pm$ 4.13& 18.9 $\pm$ 3.87 & 3.67 $\pm$ 0.78& 11.9 $\pm$ 3.00 & 9.34 $\pm$ 3.15\\
GoFAR (HER) & 28.3$\pm$ 0.65 & 19.8$\pm$ 2.82 & 20.5$\pm$ 2.29 & 3.85$\pm$ 0.80 & 8.02$\pm$ 5.70 & 10.51 $\pm$ 3.51\\ 
GoFAR (Binary) & 26.1$\pm$ 1.14 & 17.4$\pm$1.78 & 17.4 $\pm$ 2.67 & 3.69$\pm$ 1.75 & 6.01$\pm$ 1.62 & 5.13 $\pm$ 4.05\\ 
GoFAR (KL) & 0$\pm$0.0 & 0$\pm$0.0 & 0$\pm$0.0 & 0$\pm$0.0 & 0$\pm$ 0.0 & 0$\pm$ 0.0\\ 
\bottomrule
\end{tabular}}
\end{table}

\subsection{Real-World Dexterous Manipulations} 
\label{appendix:additional-results-dclaw}
In our qualitative analysis, we visualize all methods on a specific task instance of turning the valve prong (marked by the \textcolor{red}{red} strip) clockwise for $90$ degree; the goal location is marked by the \textcolor{green}{green} strip. The robot initial pose is randomized. As shown in Figure~\ref{figure:dclaw-real-vis}, GoFAR reaches the goal with three random initial poses, whereas all baselines fail. See the figure captions for detail. Policy videos are included in the supplementary material. 

\begin{figure}[H]
\begin{subfigure}[b]{\linewidth}
\includegraphics[width=\linewidth]{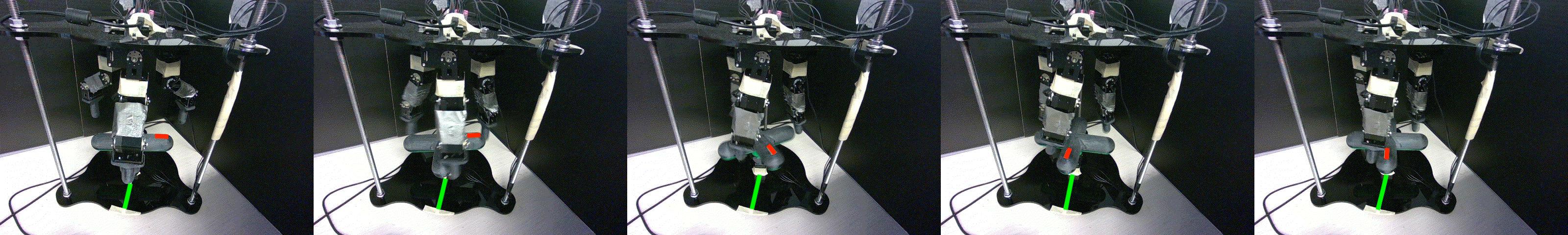} 
\includegraphics[width=\linewidth]{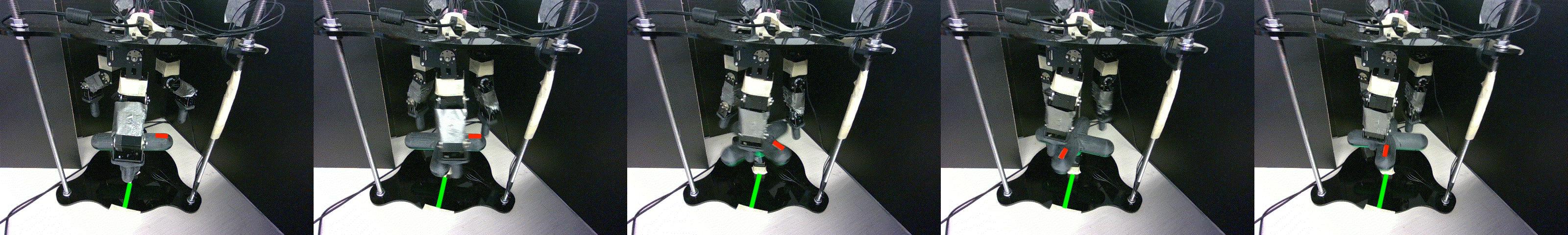}
\includegraphics[width=\linewidth]{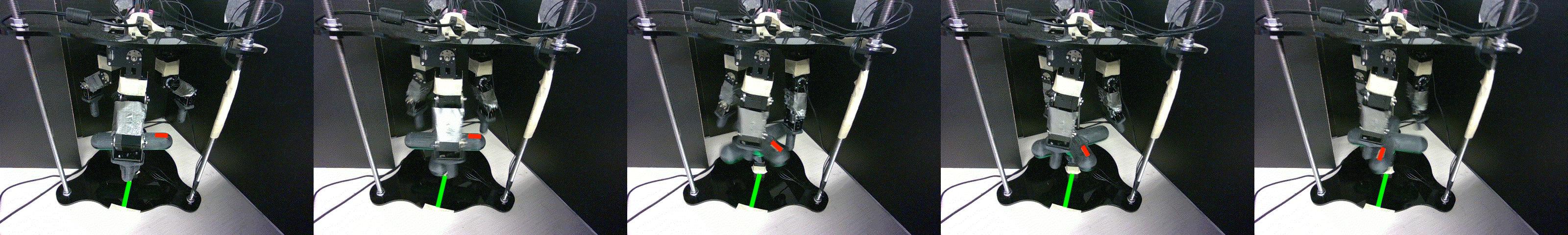}
\caption{GoFAR robustly achieved the goal with three random initial poses; in the first two runs, it demonstrates ``recovery'' behavior, as the robot would initially overshoot and then turn the valve counterclockwise. In the last run, the robot initially undershoots and then turns again to reach the goal.}
\end{subfigure}

\begin{subfigure}[b]{\linewidth}
\includegraphics[width=\linewidth]{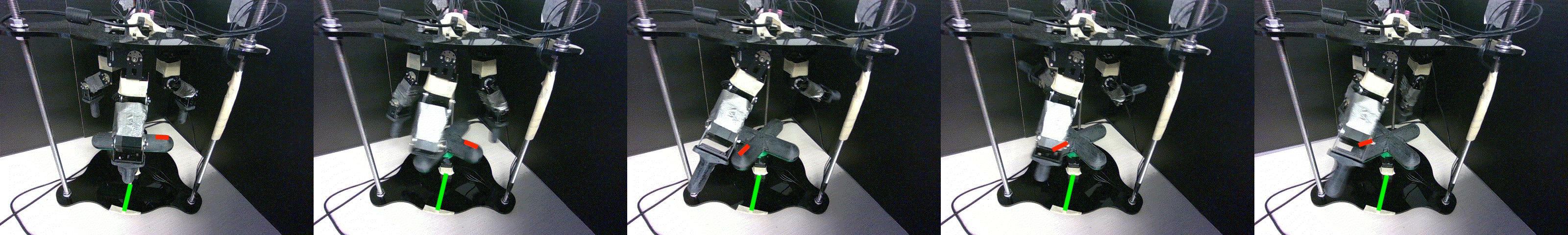} 
\includegraphics[width=\linewidth]{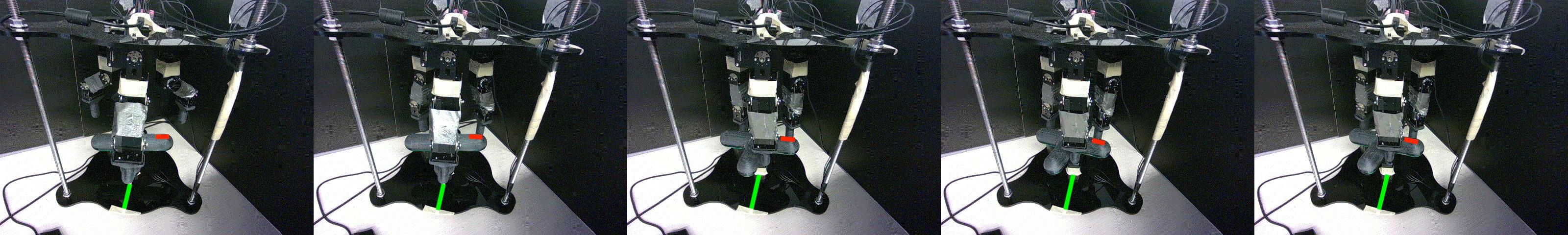}
\includegraphics[width=\linewidth]{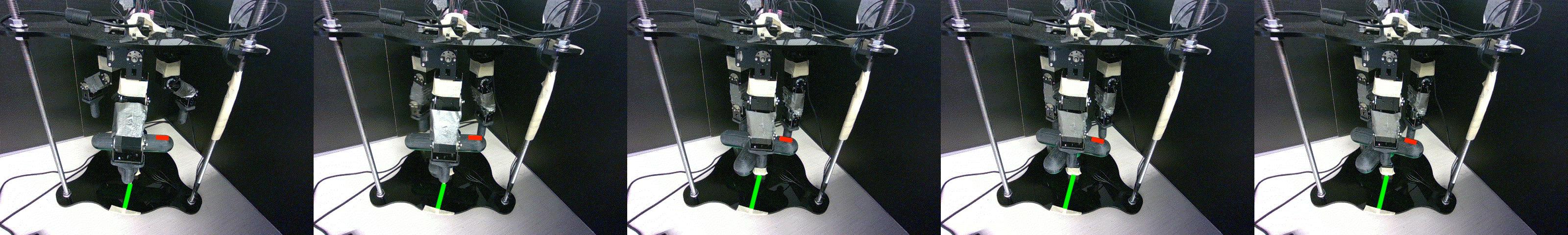}
\includegraphics[width=\linewidth]{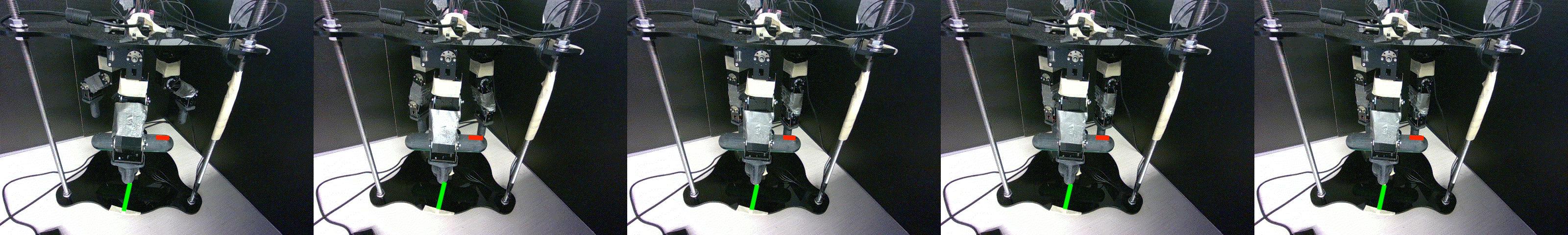}
\caption{Baselines fail to turn the volve prong (marked by the \textcolor{red}{red} strip) to the goal angle (marked by the \textcolor{green}{green} strip). AM is the only method that is able to rotate the prong to some degree, though it overshoots in this case and exhibits unnatural behavior.}
\end{subfigure}
\caption{D'ClawTurn policy visualization.}
\label{figure:dclaw-real-vis}
\end{figure} 

\subsection{Zero-Shot Plan Transfer}
\label{appendix:additional-results-zero-shot}
We visualize GoFAR hierarchical controller and the plain low-level controller on three distinct goals in Figure~\ref{figure:zero-shot-transfer-vis}. See the figure caption for detail. Policy videos are included in the supplementary material. 

\begin{figure}[H]
\begin{subfigure}[b]{\linewidth}
\includegraphics[width=\linewidth]{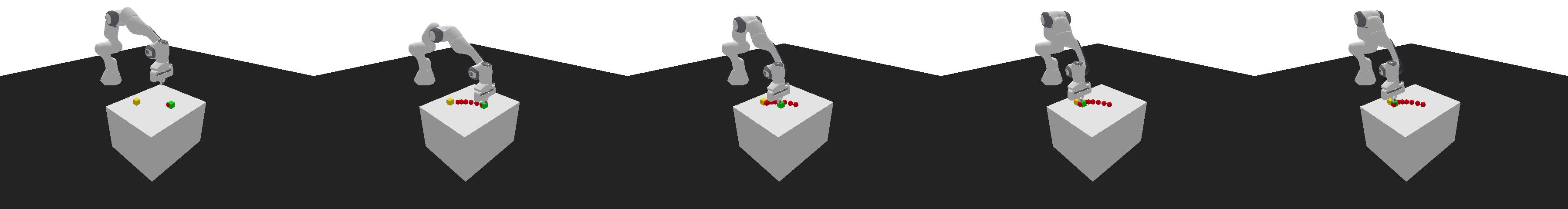} 
\includegraphics[width=\linewidth]{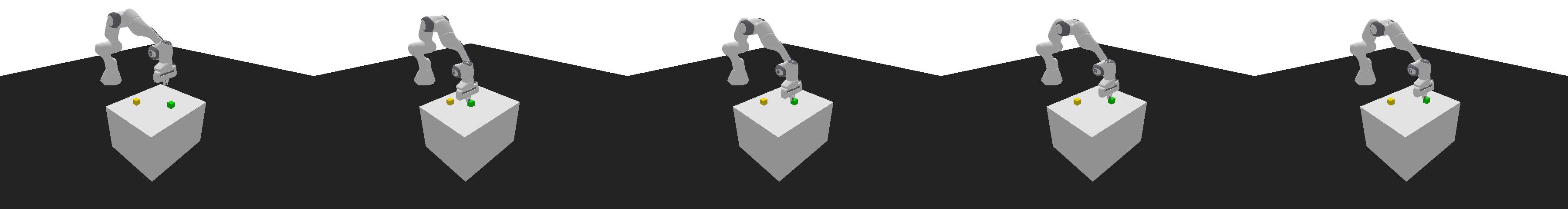}
\caption{Goal 1}

\end{subfigure}

\begin{subfigure}[b]{\linewidth}
\includegraphics[width=\linewidth]{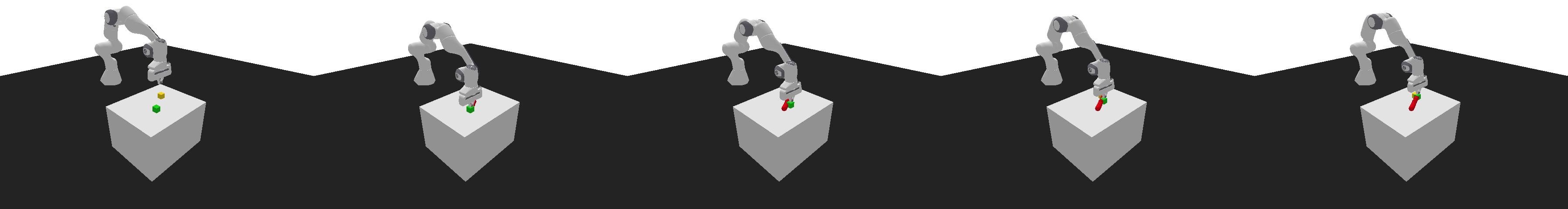} 
\includegraphics[width=\linewidth]{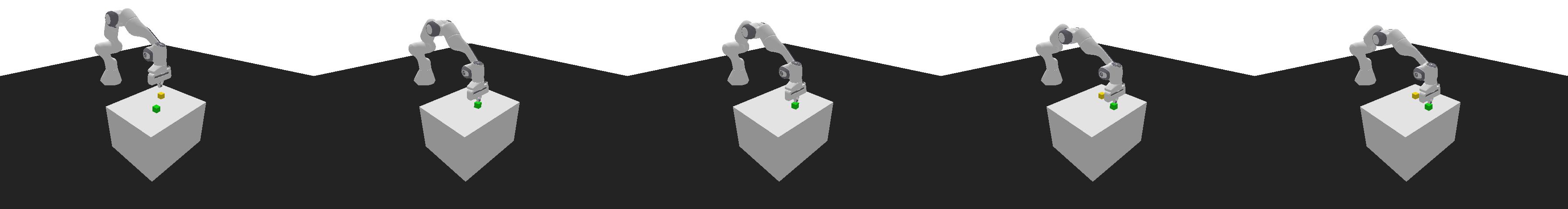}
\caption{Goal 2}
\end{subfigure}

\begin{subfigure}[b]{\linewidth}
\includegraphics[width=\linewidth]{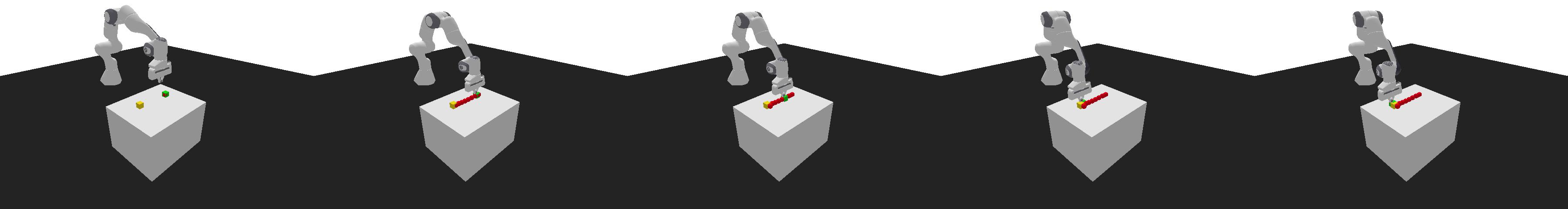} 
\includegraphics[width=\linewidth]{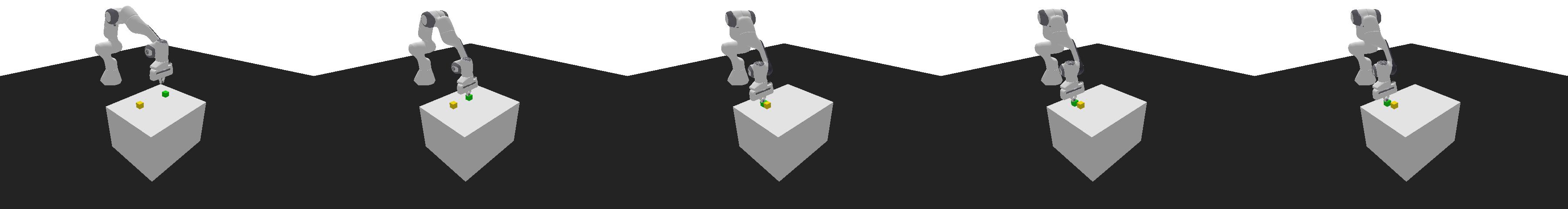}
\caption{Goal 3}
\end{subfigure}
\caption{Qualitative comparison of GoFAR hierarchical controller (top) vs. plain low-level controller (bottom) on representative goals in the Franka pushing task. \textcolor{red}{Red} circles represent intermediate subgoals generated by the GoFAR planner. As shown, the low-level controller only succeeds in Goal 3, whereas the hierarchical controller achieves the distant goals in all three cases.}
\label{figure:zero-shot-transfer-vis}
\end{figure}
\end{document}

%% file: defs.tex
\usepackage{soul}
\usepackage{enumitem}
\usepackage{gensymb}
\usepackage{textcomp}
\usepackage{url}

\usepackage[colorinlistoftodos]{todonotes}

\long\def\ignorethis#1{}

%% file: neurips_checklist.tex
\section*{Checklist}

The checklist follows the references.  Please
read the checklist guidelines carefully for information on how to answer these
questions.  For each question, change the default \answerTODO{} to \answerYes{},
\answerNo{}, or \answerNA{}.  You are strongly encouraged to include a {\bf
justification to your answer}, either by referencing the appropriate section of
your paper or providing a brief inline description.  For example:
\begin{itemize}
  \item Did you include the license to the code and datasets? \answerYes{See Section~\ref{gen_inst}.}
  \item Did you include the license to the code and datasets? \answerNo{The code and the data are proprietary.}
  \item Did you include the license to the code and datasets? \answerNA{}
\end{itemize}
Please do not modify the questions and only use the provided macros for your
answers.  Note that the Checklist section does not count towards the page
limit.  In your paper, please delete this instructions block and only keep the
Checklist section heading above along with the questions/answers below.

\begin{enumerate}

\item For all authors...
\begin{enumerate}
  \item Do the main claims made in the abstract and introduction accurately reflect the paper's contributions and scope?
    \answerYes{}
  \item Did you describe the limitations of your work?
    \answerYes{}
  \item Did you discuss any potential negative societal impacts of your work?
    \answerNA{}
  \item Have you read the ethics review guidelines and ensured that your paper conforms to them?
    \answerYes{}
\end{enumerate}

\item If you are including theoretical results...
\begin{enumerate}
  \item Did you state the full set of assumptions of all theoretical results?
    \answerYes{}
	\item Did you include complete proofs of all theoretical results?
    \answerYes{See Appendix~\ref{appendix:proof}.}
\end{enumerate}

\item If you ran experiments...
\begin{enumerate}
  \item Did you include the code, data, and instructions needed to reproduce the main experimental results (either in the supplemental material or as a URL)?
    \answerYes{We include the code and data for our simulation experiments in the supplementary material.}
  \item Did you specify all the training details (e.g., data splits, hyperparameters, how they were chosen)?
    \answerYes{See Appendix~\ref{appendix:experimental-details}.}
	\item Did you report error bars (e.g., with respect to the random seed after running experiments multiple times)?
    \answerYes{See Appendix~\ref{appendix:additional-results}.}
	\item Did you include the total amount of compute and the type of resources used (e.g., type of GPUs, internal cluster, or cloud provider)?
    \answerYes{See Appendix~\ref{appendix:experimental-details}.}
\end{enumerate}

\item If you are using existing assets (e.g., code, data, models) or curating/releasing new assets...
\begin{enumerate}
  \item If your work uses existing assets, did you cite the creators?
    \answerYes{}
  \item Did you mention the license of the assets?
    \answerNA{}
  \item Did you include any new assets either in the supplemental material or as a URL?
    \answerYes{We include videos in our supplementary material.}
  \item Did you discuss whether and how consent was obtained from people whose data you're using/curating?
    \answerNA{}
  \item Did you discuss whether the data you are using/curating contains personally identifiable information or offensive content?
   \answerNA{}
\end{enumerate}

\item If you used crowdsourcing or conducted research with human subjects...
\begin{enumerate}
  \item Did you include the full text of instructions given to participants and screenshots, if applicable?
    \answerNA{}
  \item Did you describe any potential participant risks, with links to Institutional Review Board (IRB) approvals, if applicable?
    \answerNA{}
  \item Did you include the estimated hourly wage paid to participants and the total amount spent on participant compensation?
    \answerNA{}
\end{enumerate}

\end{enumerate}